\def\year{2022}\relax
\documentclass[letterpaper]{article} 
\usepackage{aaai22}  
\usepackage{times}  
\usepackage{helvet}  
\usepackage{courier}  
\usepackage[hyphens]{url}  
\usepackage{graphicx} 
\usepackage{natbib}  
\usepackage{caption} 
\DeclareCaptionStyle{ruled}{labelfont=normalfont,labelsep=colon,strut=off} 
\usepackage{algorithm}
\usepackage{algorithmic}

\usepackage{newfloat}
\usepackage{listings}
\floatstyle{ruled}
\newfloat{listing}{tb}{lst}{}
\floatname{listing}{Listing}
\title{Balanced Self-Paced Learning for AUC Maximization}
\author{
    Bin Gu\textsuperscript{\rm 1,\rm 2}\equalcontrib,
    Chenkang Zhang\textsuperscript{\rm 2}\equalcontrib,
    Huan Xiong\textsuperscript{\rm 1,\rm 3}\equalcontrib,
    Heng Huang\textsuperscript{\rm 4}
}
\affiliations{
	\textsuperscript{\rm 1}MBZUAI, United Arab Emirates \\
	\textsuperscript{\rm 2}School of Computer \& Software, Nanjing University of Information Science \& Technology, P.R.China\\
	\textsuperscript{\rm 3}Institute for Advanced Study in Mathematics, Harbin Institute of Technology, P.R.China\\
	\textsuperscript{\rm 4}Department of Electrical \& Computer Engineering, University of Pittsburgh, PA, USA \\
	bin.gu@mbzuai.ac.ae, 20201221058@nuist.edu.cn, huan.xiong.math@gmail.com,	heng.huang@pitt.edu

}

\usepackage{bibentry}

\usepackage{amsmath}
\usepackage{amssymb}
\usepackage{mathtools}
\usepackage{amsthm}

\usepackage{multirow}
\usepackage{bbding}
\usepackage{amsmath}\allowdisplaybreaks 
\usepackage{graphicx} 
\usepackage{algorithm}
\usepackage{algorithmic}
\usepackage{setspace}
\usepackage{bm,enumitem,caption}
\usepackage{subfigure}
\usepackage{wrapfig}
\usepackage{multirow}
\newsavebox\CBox

\newtheorem{theorem}{Theorem}
\newtheorem{lemma}{Lemma} 
\newtheorem{remark}{Remark}

\newtheorem{definition}{Definition} 
 
\DeclareMathOperator*{\argmin}{arg\,min}

\begin{document}

\maketitle

\begin{abstract} 
	Learning to improve AUC performance is an important topic in machine learning. However,  AUC maximization algorithms may decrease generalization performance due to the noisy data. Self-paced learning is an effective method for handling noisy data. However, existing self-paced learning methods are limited to pointwise learning,  while AUC maximization is a pairwise learning problem. To solve this challenging problem, we innovatively propose a balanced self-paced AUC maximization algorithm (BSPAUC). Specifically, we first provide a statistical objective for self-paced AUC.
	Based on this, we propose our self-paced  AUC maximization formulation, where a novel balanced self-paced regularization term is embedded to ensure that the selected positive and negative samples have  proper  proportions. Specially, the sub-problem with respect to all weight variables may be non-convex in our formulation, while the  one is normally convex in existing self-paced problems. To address this, we propose a doubly cyclic block coordinate descent method.
	More importantly, we prove that the sub-problem with respect to all weight variables converges to a stationary point on the basis of closed-form solutions, and  our BSPAUC converges to a stationary point of our fixed optimization   objective under a mild assumption. Considering both the deep learning and kernel-based implementations, experimental results on several large-scale datasets demonstrate that our BSPAUC has a  better generalization performance than existing state-of-the-art  AUC maximization methods.
\end{abstract}
\section{Introduction}  \label{sec_Introduction}
Learning to improve AUC performance is an important topic in machine learning, especially for imbalanced datasets  \cite{huang2005using,ling2003auc,cortes2004auc}. Specifically, for severely imbalanced binary classification datasets, a classifier may achieve a high prediction accuracy if it predicts all samples to be the dominant class. However, the classifier actually has a poor generalization performance because it cannot properly classify samples from non-dominant class. AUC (area under the ROC curve) \cite{hanley1982meaning}, which measures the probability of a randomly drawn positive sample having a higher decision value than a randomly drawn negative sample \cite{mcknight2010mann}, would be a better evaluation criterion for imbalanced datasets.

Real-world data tend to be massive in quantity, but with quite a few unreliable noisy data that can lead to decreased generalization performance. Many studies have tried to address this, with some degree of success \cite{wu2007robust,safe,zhang2020self}. However, most of these studies only consider the impact of noisy data on accuracy, rather than on AUC. In many cases,  AUC maximization algorithms may decrease generalization performance due to the noisy data. Thus,  how to deal with noisy data in AUC maximization problems is still an open topic.

Since its birth  \cite{kumar2010self}, self-paced learning (SPL)  has attracted increasing attention  \cite{wan2020self,klink2020self,ghasedi2019balanced} because it can simulate the learning principle of humans, \emph{i.e.}, starting with  easy samples and then gradually introducing more complex samples into training. Complex samples  are considered to own larger loss than easy samples, and noise  samples normally have a  relatively  large loss. Thus,  SPL could  reduce the importance  weight of noise  samples because they are treated as complex samples. Under the  learning paradigm of SPL, the model is constantly corrected and  its robustness 
is improved.  Thus, SPL is  an effective method for handling noisy data. Many experimental and theoretical analyses have proved its robustness \cite{meng2017theoretical,liu2018understanding,zhang2020self}.  However, existing SPL methods are limited to pointwise learning,  while AUC maximization is a pairwise learning problem.  

To solve this challenging problem, we innovatively propose a balanced self-paced AUC maximization algorithm (BSPAUC).
Specifically, we first provide an upper bound of expected AUC risk  by the  empirical AUC  risk on training samples obeying pace distribution plus two  more terms  related to SPL. Inspired by this, we propose our balanced self-paced AUC maximization formulation.  In particular, the sub-problem with respect to all weight variables may be non-convex in our formulation, while the  one is normally convex in existing self-paced problems. To solve this challenging difficulty, we propose a doubly cyclic block coordinate descent method to optimize our formulation.

The main contributions of this paper are summarized as follows.
\begin{enumerate}[leftmargin=0.2in]
	\setlength{\parsep}{0ex} 
	\setlength{\topsep}{0ex} 
	\setlength{\itemsep}{0ex} 
	\item Inspired by our statistical explanation for self-paced AUC,  we  propose a balanced self-paced AUC maximization formulation with a  novel balanced self-paced regularization term. To the best of our knowledge, this is the first objective formulation introducing  SPL into the AUC maximization problem.
	\item We propose a doubly cyclic block coordinate descent method to optimize our formulation. Importantly, we give closed-form solutions of the two  weight variable blocks and provide two instantiations of optimizing the  model parameter block  on  kernel learning and deep learning, respectively.
	\item We prove that the sub-problem with respect to all weight variables converges to a stationary point on the basis of closed-form solutions, and our BSPAUC converges to a stationary point of our fixed optimization   objective under a mild assumption.
\end{enumerate}

\section{Self-Paced AUC}\label{sec_stat}

In this section, we first provide a statistical objective for self-paced AUC. Inspired by this, we provide our objective.

\subsection{Statistical Objective}

\textbf{Empirical and Expected AUC Objective for IID Data:}
Let $X$ be a compact subset of $\mathbb{R}^d$, $Y = \{-1, +1\}$ be the label set and $Z = X \times Y$. Given a distribution $P(z)$ and let  $S=\{z_i=(x_i,y_i)\}_{i=1}^n$ be an independent and identically distributed (IID)  training set drawn from $P(z)$, where $x_i \in X$, $y_i \in Y$ and $z_i \in Z$. Thus, empirical  AUC risk on $S$ can be formulated as:
\begin{align} \label{emp}
	R_{emp}(S;f)=\frac{1}{n(n-1)}\sum_{z_i,z_j \in S, z_i \neq z_j} L_f(z_i,z_j).
\end{align}
Here, $f \in \mathcal{F}: \mathbb{R}^d \to \mathbb{R}$ is one real-valued function and the pairwise loss function $L_f(z_i,z_j)$ for AUC is defined as:
\begin{equation*} 
	L_f(z_i,z_j)=
	\left \{\begin{array} {l@{\ \ \textrm{if}  \ \ }l} 0   &y_i=y_j
		\\ \mathbb{I}(f(x_i) \le f(x_j))  &y_i =+1 \& \ y_j=-1\\
		\mathbb{I}(f(x_j) \le f(x_i))  &y_j =+1 \& \ y_i=-1
	\end{array} \right.
\end{equation*}
where $\mathbb{I}(\cdot)$ is the indicator function such that
$\mathbb{I}(\pi)$ equals 1 if $\pi$ is true and 0 otherwise. Further, the expected AUC risk  for the  distribution $P(z)$ can be defined as:
\begin{align}\label{exp}
	R_{exp}(P(z);f):=&\mathbb{E}_{z_1,z_2 \sim P(z)^2 } L_f(z_1,z_2) \nonumber  \\
	=&\mathbb{E}_{S}[R_{emp}(S;f)].
\end{align}

\noindent \textbf{Compound Data:} \ In practice, it is  expensive to collect completely pure dataset because that would involve domain experts to evaluate the quality of collected data. Thus, it is a reasonable assumption that our training set in reality is composed of not only clean target  data but also  a  proportion of noise samples \cite{natarajan2013learning,kang2019robust}. If we denote  the distribution of  clean target data by $P_{target}(z)$,  and  one of noisy data by $P_{noise}(z)$, the distribution for the real training data can be formulated as $P_{train}(z)=\alpha P_{target}(z)+(1-\alpha)P_{noise}(z)$, where $\alpha \in [0,1]$ is a  weight to balance  $P_{target}(z)$ and $P_{noise}(z)$. We also illustrate this compound training data  in Figure \ref{data}. Note that we assume  noise samples normally have a relatively large loss, and thus they are  treated as complex samples in SPL as discussed previously.


\begin{figure}[t] 
	\centering
	\includegraphics[width=0.45\textwidth]{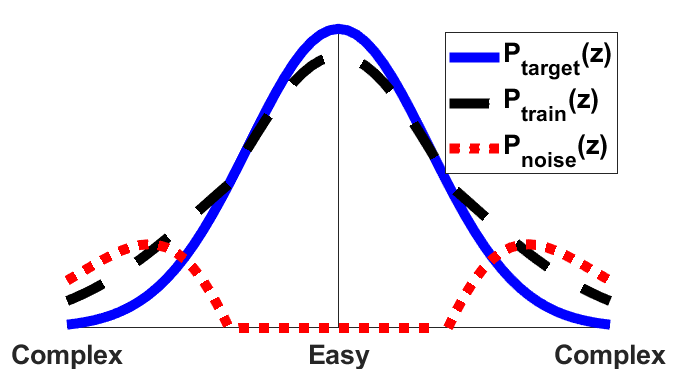}
	\caption{Data distribution on the degree of complexity.} \label{data}
\end{figure}

\noindent \textbf{Upper Bound of Expected AUC for Compound Data:} \
Gong et al. \cite{why} connect the distribution $P_{train}(z)$ of the training set  with the distribution $P_{target}(z)$ of the target set using a  weight function $W_{\lambda}(z)$:
\begin{align}  \label{OriginalF} 
	P_{target}(z)=\frac{1}{\alpha_*}W_{\lambda}(z)P_{train}(z), 
\end{align}
where  $0 \le W_{\lambda}(z) \le 1 $ and $\alpha_*=\int_{Z}W_{\lambda}(z)P_{train}(z)dz$  denotes the normalization factor.  Intuitively, $W_{\lambda}(z)$ gives larger weights to easy samples than to complex samples and with the increase of pace parameter $\lambda$, all samples tend to be assigned larger weights.

Then, Eq. (\ref{OriginalF}) can be  reformulated as:
\begin{align} \label{NewF}
	& P_{train}(z)=\alpha_*P_{target}(z)+(1-\alpha_*)E(z), \\
	& E(z)=\frac{1}{1-\alpha_*}(1-W_{\lambda_*}(z))P_{train}(z). \nonumber
\end{align}
Here, $E(z)$ is related to $P_{noise}(z)$. Based on (\ref{NewF}), we define the pace distribution $Q_{\lambda}(z)$ as:
\begin{align} \label{Q}
	Q_{\lambda}(z)=\alpha_{\lambda}P_{target}(z)+(1-\alpha_{\lambda})E(z),
\end{align}
where $\alpha_{\lambda}$ varies from $1$ to $\alpha_{*}$ with increasing pace parameter $\lambda$.  Correspondingly,  $Q_{\lambda}(z)$ simulates the changing process from $P_{target}(z)$ to $P_{train}(z)$.  Note that $Q_{\lambda}(z)$ can also be regularized into the following formulation:
\begin{align*}   
	Q_{\lambda}(z)  \propto  W_{\lambda}(z)P_{train}(z), 
\end{align*}
where $W_{\lambda}(z)$  through normalizing its maximal value to 1.

We derive the following result on the upper bound of the expected AUC risk. Please refer to Appendix for the proof.
\begin{theorem} \label{theoremUp}
	For any $\delta>0$ and  any $ f \in \mathcal{F}$, with confidence at least $1-\delta$ over a training set $S$, we have:
	\begin{align} \label{analysis}
		&R_{exp}(P_{target};f)  \nonumber \\
	  \leq & \frac{1}{n_{\lambda}(n_{\lambda}-1)} \sum_{z_i,z_j \in S \atop z_i\neq z_j} W_{\lambda}(z_i) W_{\lambda}(z_j) L_f(z_i,z_j)  \nonumber\\
		& +\sqrt{\frac{\ln (1/\delta)}{n_{\lambda}/2}}   +  e_{\lambda}
	\end{align}
	where $n_{\lambda}$ denotes the number of selected samples from the training set and $e_{\lambda}:= R_{exp}(P_{target};f) - R_{exp}(Q_{\lambda};f)$ decreases monotonically from $0$  with the increasing of $\lambda$.
\end{theorem}
We will give a detailed  explanation on the three terms of the upper bound (\ref{analysis}) for Theorem \ref{theoremUp} as follows.
\begin{enumerate}[leftmargin=0.2in]
	\item The first term corresponds to the empirical AUC  risk on training samples obeying pace distribution $Q_{\lambda}$. With increasing $\lambda$, the weights $W_{\lambda}(z)$ of complex samples gradually increase and these complex samples are gradually involved in training. 
	\item The second term reflects the expressive capability of training samples on the pace distribution $Q_{\lambda}$. With  increasing $\lambda$, more samples are considered, the  pace distribution $Q_{\lambda}$ can be expressed better.
	\item The last term measures the generalization capability of the learned model. As shown in Eq. (\ref{Q}), with increasing $\lambda$, $\alpha_{\lambda}$ gets smaller and the generalization of the learned model becomes worse. This is due to the gradually more evident deviation $E(z)$ from  $Q_{\lambda}$ to $P_{target}$.
\end{enumerate}

Inspired by the upper bound  (\ref{analysis}) and the  above explanations,  we will propose our self-paced AUC maximization formulation in the next subsection.


\subsection{Optimization Objective}

First of all, we give the definition of some necessary notations.  Let  $\theta$  represent the model parameters,   $n$ and $m$  denote the number of positive and negative samples respectively,  $\mathbf{v} \in [0,1]^n$ and   $\mathbf{u} \in [0,1]^m$ be the weights of positive and negative samples respectively,   $\lambda$ be the pace parameter for controlling the learning pace, and $\mu$ balance  the proportions of  selected  positive and negative samples. The zero-one loss is replaced by the  pairwise hinge loss which is  a common surrogate loss in AUC maximization problems \cite{brefeld2005auc,zhao2011online,gao2015consistency}. Then, inspired by the upper bound  (\ref{analysis}), we have the following optimization objective:
\begin{align}  \label{BSPAUCOF} 
	& \mathcal{L}(\theta,\mathbf{v},\mathbf{u};\lambda)  \nonumber \\ 
	=&   \underbrace{\frac{1}{nm} \sum_{i=1}^{n}\sum_{j=1}^{m}v_i u_j \xi_{ij}}_{\mathbf{1}}  \underbrace{-  \lambda \left(\frac{1}{n}\sum_{i=1}^{n} v_i+\frac{1}{m}\sum_{j=1}^{m} u_j \right)}_{\mathbf{2}} \nonumber  \\
	& \underbrace{ + \tau \Omega(\theta)}_{\mathbf{3}}   \underbrace{+ \mu \left(\frac{1}{n}\sum_{i=1}^{n} v_i-\frac{1}{m}\sum_{j=1}^{m} u_j \right)^2}_{\mathbf{4}}        \\
	& \ s.t.  \ \mathbf{v}\in [0,1]^n,\mathbf{u}\in [0,1]^m    \nonumber
\end{align}where $\xi_{ij}=\max \{1-f(x^+_i)+f(x^-_j), 0 \}$ is the pairwise hinge loss and $\Omega(\theta)$ is the regularization term to  avoid overfitting. Specifically, $\theta$ is a matrix composed of weights and biases of each layer and $\Omega(\theta)$ is formulated as $\frac{1}{2}||\theta||_F^2$ in the deep learning setting and in the kernel-based setting, $\Omega(\theta)$ is formulated as $\frac{1}{2}||\theta||_\mathcal{H}^2$, where $||\cdot||_{\mathcal{H}}$ denotes the norm in a reproducing kernel Hilbert space (RKHS) $\mathcal{H}$.

As the explanation about the upper bound (\ref{analysis}) shows,  the upper bound is composed of  three aspects:  empirical risk,  sample expression ability and  model generalization ability. Inspired by this, we construct our optimization objective (\ref{BSPAUCOF}) which  also considers the above three aspects. Specifically,  the term \textbf{1} in Eq. (\ref{BSPAUCOF}) corresponds to  the (weighted) empirical risk . The term \textbf{2} in Eq. (\ref{BSPAUCOF}) corresponds to the  sample expression ability.  As we explained before,  sample expression ability is related to the number of selected samples and  pace parameter $\lambda$ in term \textbf{2} is used   to control the number of selected samples.  The term \textbf{3} in Eq. (\ref{BSPAUCOF}) corresponds to the  model generalization ability which is  a common model regularization  term  used to avoid model overfitting.

In  addition, the   term \textbf{4} in  Eq. (\ref{BSPAUCOF}) is our new proposed balanced self-paced regularization term, which is used to balance  the proportions of  selected  positive and negative samples as   Figure \ref{BorNot}  shows. Specifically, if   this term is not used, only the degree of sample complexity is considered and this would lead to severe  imbalance between the proportions of selected positive and negative samples in practice. But the proportions of selected positive and negative samples  could be ensured properly if   this term is enforced.

\begin{figure}[t] 
	\centering
	\includegraphics[width=0.4\textwidth]{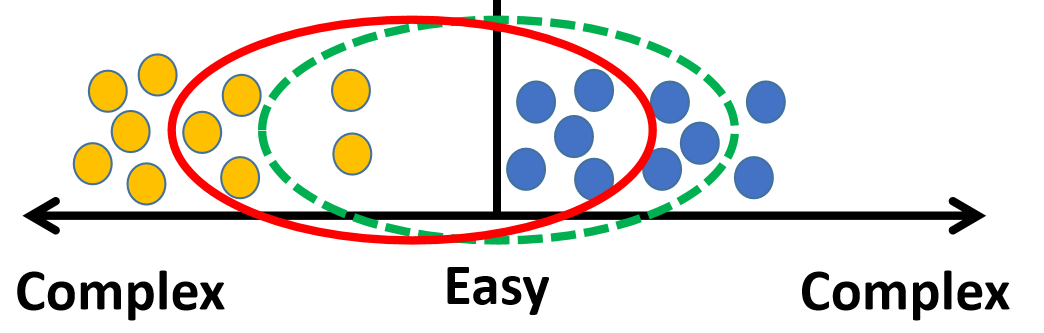}
	\caption{Contrast between whether or not using the balanced self-paced regularization term. (The yellow ball represents the positive sample, and the blue ball represents the negative sample.  The green dotted ellipsoid  represents the selected samples without using the term, and the red solid ellipsoid denotes the selected samples by using the term.)}
	\label{BorNot}
	
\end{figure}

\section{BSPAUC  Algorithm}\label{sec_alg}

In this section, 
we  propose our BSPAUC algorithm (\emph{i.e.}, Algorithm \ref{alg1}) to solve the problem (\ref{BSPAUCOF}). Different from  traditional SPL algorithms \cite{wan2020self,klink2020self,ghasedi2019balanced} which have two blocks of variables, our problem (\ref{BSPAUCOF}) has three blocks of variables which makes the optimization process more challenging. To address this issue, we propose a doubly cyclic block coordinate descent algorithm as shown in Algorithm \ref{alg1}, which consists of two layers of  cyclic block coordinate descent algorithms. The outer layer cyclic block coordinate descent procedure (\emph{i.e.}, lines 2-9 of Algorithm \ref{alg1}) follows the general optimization procedure of SPL to  optimize  all weight variables  and model parameters alternatively. The inner layer cyclic block coordinate descent procedure (\emph{i.e.}, lines 3-6 of Algorithm \ref{alg1}) is aimed to optimize the two blocks of weight variables (\emph{i.e.}, $\mathbf{v}$ and $\mathbf{u}$)  alternatively.

In the following, we revolve around the outer layer cyclic block coordinate descent procedure to discuss how to optimize all weight variables (\emph{i.e.}, $\mathbf{v}$ and $\mathbf{u}$) and model parameters  (\emph{i.e.}, $\theta$) respectively.

%

\begin{algorithm}   [htb]
	\caption{Balanced self-paced learning for AUC maximization}
	\begin{algorithmic}[1]   \label{BSPAUC}
		\REQUIRE The training set, $\theta^0$, $T$, $\lambda^0$, $c$, $\lambda_{\infty}$ and $\mu$.\\ 
		\STATE Initialize $\mathbf{v}^0= \mathbf{1}_n$ and $\mathbf{u}^0= \mathbf{1}_m$.
		\FOR { $t=1, \cdots ,T$}
		\REPEAT
		\STATE Update $\mathbf{v}^{t}$ through  Eq. (\ref{solutionofv}).
		\STATE Update $\mathbf{u}^{t}$ through  Eq. (\ref{solutionofu}).
		\UNTIL{Converge to a  stationary point.}
		\STATE Update $\theta^{t}$ through solving (\ref{f}).
		\STATE $\lambda^{t}=\min \{ c\lambda^{t-1},\lambda_{\infty} \}$.
		\ENDFOR
		\ENSURE The model solution $\theta$.\\ 
	\end{algorithmic}
	\label{alg1}
\end{algorithm}
\begin{algorithm}  [htb]
	\caption{Deep learning implementation of solving Eq.  (\ref{f})}
	\begin{algorithmic}[1]    \label{Deeplearning}
		\REQUIRE $T,\mathbf{\eta},\tau, \hat{X}^+, \hat{X}^-, \pi,\theta^0$.
		\FOR { $i=1, \cdots ,T$}
		\STATE Sample $\hat{x}^+_1,...,\hat{x}^+_\pi$ from $\hat{X}^+$.
		\STATE Sample $\hat{x}^-_1,...,\hat{x}^-_\pi$ from $\hat{X}^-$.
		\STATE  Calculate $f({x})$.
		\STATE Update $\theta$ by the following formula:
		\begin{align*}  
			\theta = &(1-\eta_i\tau)\theta
			-\frac{\eta_i}{\pi}\sum_{j=1}^\pi v_j u_j \frac{ \partial \xi_{jj}} { \partial \theta } . \qquad \qquad
		\end{align*}
		\ENDFOR
		\ENSURE $\theta$.
	\end{algorithmic}
\end{algorithm}

\subsection{Optimizing $\mathbf{v}$ and $\mathbf{u}$}
Firstly, we consider the sub-problem with respect to  all weight variables  which is normally convex in existing self-paced problems. However, if we fix $\theta$ in Eq. (\ref{BSPAUCOF}), the sub-problem   with respect to  $\mathbf{v}$ and $\mathbf{u}$  could be non-convex as shown in Theorem \ref{Non-con}.   Please refer to Appendix for the proof of Theorem \ref{Non-con}.
\begin{theorem} \label{Non-con}
	If we fix $\theta$ in Eq. (\ref{BSPAUCOF}), the sub-problem   with respect to $\mathbf{v}$ and $\mathbf{u}$  maybe  non-convex.
\end{theorem}
In order to  address the  non-convexity of sub-problem, we further divide all weight variables into two disjoint  blocks, \emph{i.e.},  weight variables $\mathbf{v}$ of positive samples and  weight variables $\mathbf{u}$ of negative samples. Note that the sub-problems \emph{w.r.t.} $\mathbf{v}$ and $\mathbf{u}$ respectively are convex. Thus,  we can  solve the following two convex sub-problems to update $\mathbf{v}$ and $\mathbf{u}$  alternatively
\begin{align}  
	\mathbf{v}^{t}= \argmin \limits_{\mathbf{v} \in [0,1]^n}  \  \mathcal{L}(\mathbf{v};\theta,\mathbf{u},\lambda) ,   \label{v} \\
	\mathbf{u}^{t}= \argmin \limits_{\mathbf{u} \in [0,1]^m}  \  \mathcal{L}(\mathbf{u};\theta,\mathbf{v},\lambda).    \label{u}
\end{align}
We derive the closed-form  solutions of the optimization problems  (\ref{v}) and (\ref{u}) respectively in the following theorem. Note that  sorted index represents the index of the sorted loss set $\{l_1, l_2, \dots \}$ which satisfies $l_i \leq l_{i+1}$.  Please refer to Appendix for the detailed proof.


\begin{theorem} \label{solutionofvu}
	The following formula gives one global optimal solution for problem  \eqref{v}:  
	\begin{equation} \label{solutionofv} 
		\left \{\begin{array} {ll} 
			&v_p=1   { \ \ \textrm{if}  \ \ }   l^+_p <  \lambda - 2\mu \left(\frac{p}{n}-\frac{\sum_{j=1}^m u_j}{m} \right)
			\\ &v_p=n  \left( \frac{\sum_{j=1}^m u_j}{m}-\frac{l^+_p - \lambda}{2 \mu}-\frac{p-1}{n} \right)  {\textrm{otherwise} }  
			\\
			&v_p=0   {\ \ \textrm{if}  \ \ }  l^+_p > \lambda - 2\mu \left(\frac{p-1}{n}-\frac{\sum_{j=1}^m u_j}{m} \right) 
		\end{array} \right.
	\end{equation}
	where  $p \in \{1,...,n\}$ is the sorted index based on the loss values $l^+_p=\frac{1}{m}\sum_{j=1}^m u_j\xi_{pj}$. \\
	The following formula gives one global optimal solution for problem \eqref{u}:  
	\begin{equation} \label{solutionofu} 
		\left \{\begin{array} {ll}
			&u_q=1   {\ \ \textrm{if}  \ \ }   l^-_q <  \lambda - 2\mu \left(\frac{q}{m}-\frac{\sum_{i=1}^n v_i}{n}\right) 
			\\
			&u_q=m \left(\frac{\sum_{i=1}^n v_i}{n}-\frac{l^-_q - \lambda}{2 \mu}-\frac{q-1}{m}\right) {  \textrm{otherwise}}  
			\\
			&u_q=0   {\ \ \textrm{if}  \ \ } l^-_q > \lambda - 2\mu \left(\frac{q-1}{m}-\frac{\sum_{i=1}^n v_i}{n}\right) 
		\end{array} \right.
	\end{equation}
	where  $q \in \{1,...,m\}$ is the sorted index based on the loss values $l^-_q=\frac{1}{n}\sum_{i=1}^n v_i\xi_{iq}$.
\end{theorem}
In this case, the solution (\ref{solutionofv}) of problem ({\ref{v}}) implies our advantages. Obviously, a sample with a loss greater/less than the threshold, \emph{i.e.}, $ \lambda - 2\mu (\frac{p-1}{n}-\frac{\sum_{j=1}^m u_j}{m})$ is ignored/involved in current training. In particular, the threshold is also a function of the sorted index, and consequently  decreases as the sorted index  increases. In this case, easy samples with less loss are given more preference. Besides, the proportion, \emph{i.e.}, $\frac{\sum_{j=1}^m u_j}{m}$  of selected negative samples  also affects the threshold. This means that the higher/lower the proportion of  selected negative  samples is, the more/fewer positive samples will be assigned high weights. Because of this, the proportions of selected positive and negative samples can be guaranteed to be balanced.  What's more, we can easily find that the solution (\ref{solutionofu}) of problem ({\ref{u}}) yields similar conclusions. In summary, our algorithm can not only give preference to  easy samples, but also ensure that the selected positive and negative samples have  proper  proportions.

\subsection{Optimizing $\theta$}

In this step, we fix $\mathbf{v},\mathbf{u}$ to update $\theta$:
\begin{equation} 
	\label{f}
	\theta^{t}= \argmin \limits_{ \theta}  \frac{1}{nm} \sum_{i=1}^{n}\sum_{j=1}^{m}v_i u_j \xi_{ij}
	+\tau \Omega(\theta)  + \text{const},
\end{equation}
where $\xi_{ij}=\max \{ 1-f(x^+_i)+f(x^-_j),0 \} $. Obviously, this is a  weighted AUC maximization problem. We provide two instantiations of optimizing the problem on kernel learning and deep learning settings, respectively.

For the deep learning implementation, we  compute the gradient on random pairs of weighted  samples which  are selected  from two subsets  $\hat{X}^+$ and $\hat{X}^-$  respectively.  $\hat{X}^+$ is a set of selected positive samples with weights $\hat{x}^+=(v_i,x_i^+), \forall v_i >0$ and $\hat{X}^-$ is a set of selected negative samples with weights $\hat{x}^-=(u_j, x_j^-), \forall u_j >0$.  In this case, we introduce the weighted batch AUC loss:
\begin{align*}
	\sum_{j=1}^\pi v_j u_j  \xi_{jj} =\sum_{j=1}^\pi v_j u_j  \max \{1-f(x^+_j)+f(x^-_j), 0 \} ,
\end{align*}and obtain Algorithm \ref{Deeplearning} by applying the  doubly stochastic gradient descent  method (DSGD) \cite{gu2019scalable} \emph{w.r.t.} random  pairs of  weighted samples, where $\eta$ means the learning rate.

For the kernel-based implementation,  we apply random Fourier feature method to approximate the kernel function \cite{rahimi2008random,dai2014scalable} for  large-scale problems. The mapping function of $D$ random features is defined as
\begin{align*}
	\phi_{\omega}(x)=\sqrt{1/D}[ &\cos(\omega_1x),\ldots,\cos(\omega_Dx),  \\ &\sin(\omega_1x),\ldots,\sin(\omega_Dx)]^T, 
\end{align*}where  $\omega_i$ is randomly sampled according to the density function $p(\omega)$ associated with $k(x,x')$ \cite{odland2017fourier}.  Then, based on the  weighted batch AUC loss  and the random feature mapping function $\phi_{\omega}(x)$, we obtain Algorithm 3  by applying the  triply stochastic gradient descent  method (TSGD) \cite{TSAM} \emph{w.r.t.}  random  pairs of weighted samples and random features which can be found in Appendix.

\section{Theoretical Analysis}


In this section, we  prove the convergence of our algorithms and all the proof details are available in  Appendix.

For the sake of clarity,  we define $\mathcal{K}(\mathbf{v},\mathbf{u})=\mathcal{L}(\mathbf{v},\mathbf{u};\theta,\lambda)$ as the sub-problem of (\ref{BSPAUCOF}) where $\theta$ and $\lambda$ are fixed, and then  prove that  $\mathcal{K}$  converges to  a stationary point based on the closed-form solutions (\emph{i.e.}, Theorem \ref{solutionofvu}).
\begin{theorem} \label{theormKstation}
	With the inner layer cyclic block coordinate descent procedure (\emph{i.e.}, lines 3-6 of Algorithm \ref{alg1}), the sub-problem $\mathcal{K}$ with respect to all weight variables converges to a   stationary point.
\end{theorem}
Next, we prove that  our BSPAUC converges along with the increase of hyper-parameter $\lambda$ under a mild assumption.
\begin{theorem} \label{Converge}
	If   Algorithm \ref{Deeplearning} or Algorithm 3 (in Appendix) optimizes $\theta$  such that $\mathcal{L}(\theta^{t+1};\mathbf{v}^{t+1},\mathbf{u}^{t+1},\lambda^{t}) \le \mathcal{L}(\theta^{t};\mathbf{v}^{t+1},\mathbf{u}^{t+1},\lambda^{t})$, BSPAUC  converges  along  with the increase of hyper-parameter $\lambda$.
\end{theorem}
\begin{remark}
	No matter the sub-problem (\ref{f}) is convex or not, it is a basic requirement for a solver (e.g., Algorithm   \ref{Deeplearning} and Algorithm 3 in Appendix) such that the solution $\theta^{t+1}$ satisfies:
	\begin{align*}
		\mathcal{L}(\theta^{t+1};\mathbf{v}^{t+1},\mathbf{u}^{t+1},\lambda^{t}) &\le \mathcal{L}(\theta^{t};\mathbf{v}^{t+1},\mathbf{u}^{t+1},\lambda^{t}).
	\end{align*}
	Thus, we can have that  our BSPAUC  converges  along with  the increase of hyper-parameter $\lambda$.
\end{remark}

Considering  that the hyper-parameter $\lambda$ reaches its maximum  $\lambda_{\infty}$, we will have that our BSPAUC    converges to a stationary point of $\mathcal{L}(\theta,\mathbf{v},\mathbf{u};\lambda_{\infty})$ if the iteration number $T$ is large enough.
\begin{theorem} \label{ConvergeToStan}
	If   Algorithm \ref{Deeplearning} or Algorithm 3 (in Appendix) optimizes $\theta$  such that $\mathcal{L}(\theta^{t+1};\mathbf{v}^{t+1},\mathbf{u}^{t+1},\lambda^{t}) \le \mathcal{L}(\theta^{t};\mathbf{v}^{t+1},\mathbf{u}^{t+1},\lambda^{t})$, and $\lambda$ reaches its maximum  $\lambda_{\infty}$, we will have that our  BSPAUC  converges to a stationary point of $\mathcal{L}(\theta,\mathbf{v},\mathbf{u};\lambda_{\infty})$ if the iteration number $T$ is large enough.
\end{theorem}

\begin{table}[H] 
	\centering
	\begin{tabular}{c|c|c|c}
		\hline
		$\mathbf{Dataset}$ & $\mathbf{Size}$ & $\mathbf{Dimensions}$ & $\mathbf{N_- \backslash N_+}$ \\ \hline
		sector             & 9,619           & 55,197                 & 95.18                         \\
		rcv1               & 20,242          & 47,236                 & 1.07                          \\
		a9a                & 32,561          & 123                   & 3.15                          \\
		shuttle            & 43,500          & 9                     & 328.54                        \\
		aloi               & 108,000         & 128                   & 999.00                        \\
		skin$\_$nonskin       & 245,057         & 3                     & 3.81                          \\
		cod-rna            & 331,152         & 8                     & 2.00                          \\
		poker              & 1,000,000       & 10                    & 701.24                        \\ \hline
	\end{tabular}
	\caption{Datasets. ($N_{-}$ means the number of negative samples and $N_{+}$ means the number of positive samples.)}  \label{Datasets}
\end{table}

\section{Experiments} 

In this section, we  first describe the experimental setup, and then provide our experimental results and discussion.


\begin{table*}[htb] 
	\centering
	\setlength{\tabcolsep}{4.5pt} 
	\begin{tabular}{c|c|c|c|c|c|c|c}
		\hline
		\multirow{2}{*}{Datasets} & \multicolumn{4}{c|}{Non Deep Learning Methods}                                                   & \multicolumn{3}{c}{Deep Learning Methods}                                \\ \cline{2-8} 
		& \textbf{BSPAUC}          & \textbf{TSAM}            & \textbf{KOIL}$_{FIFO++}$ & \textbf{OPAUC}  & \textbf{BSPAUC}                   & \textbf{DSAM}                     & \textbf{PPD}$_{SG}$ \\ \hline
		sector                    & \textbf{0.991$\pm$0.005} & 0.986$\pm$0.005          & 0.953$\pm$0.014          & 0.971$\pm$0.018 & \textbf{0.991$\pm$0.002} & 0.978$\pm$0.007          & 0.935$\pm$0.008     \\
		rcv1                      & \textbf{0.979$\pm$0.001} & 0.970$\pm$0.001          & 0.913$\pm$0.018          & 0.966$\pm$0.012 & \textbf{0.993$\pm$0.001} & 0.990$\pm$0.001          & 0.988$\pm$0.001     \\
		a9a                       & \textbf{0.926$\pm$0.008} & 0.904$\pm$0.015          & 0.858$\pm$0.020          & 0.869$\pm$0.013 & \textbf{0.927$\pm$0.005} & 0.908$\pm$0.003          & 0.906$\pm$0.001     \\
		shuttle                   & \textbf{0.978$\pm$0.002} & 0.970$\pm$0.004          & 0.948$\pm$0.010          & 0.684$\pm$0.036 & \textbf{0.994$\pm$0.003} & 0.989$\pm$0.004          & ---                 \\
		aloi                      & \textbf{0.999$\pm$0.001} & \textbf{0.999$\pm$0.001} & \textbf{0.999$\pm$0.001} & 0.998$\pm$0.001 & \textbf{0.999$\pm$0.001} & \textbf{0.999$\pm$0.001} & ---                 \\
		skin$\_$nonskin           & \textbf{0.958$\pm$0.004} & 0.946$\pm$0.004          & ---                      & 0.943$\pm$0.007 & \textbf{0.999$\pm$0.001} & \textbf{0.999$\pm$0.001} & 0.949$\pm$0.001     \\
		cod-rna                   & \textbf{0.973$\pm$0.006} & 0.966$\pm$0.010          & ---                      & 0.924$\pm$0.024 & \textbf{0.994$\pm$0.001} & 0.992$\pm$0.001          & 0.988$\pm$0.001     \\
		poker                     & \textbf{0.934$\pm$0.013} & 0.901$\pm$0.021          & ---                      & 0.662$\pm$0.025 & \textbf{0.990$\pm$0.004} & 0.976$\pm$0.015          & ---                 \\ \hline
	\end{tabular}
	\caption{Mean AUC results with the corresponding standard deviation on original benchmark datasets. ('–' means out of memory or  unable to handle severely imbalanced datasets.)}  \label{AUCR}
\end{table*}
\begin{table*}[htb] 
	\centering
	\begin{tabular}{c|ccc|ccc|ccc|ccc}
		\hline
		Datasets        & \multicolumn{3}{c|}{rcv1}                              & \multicolumn{3}{c|}{a9a}                               & \multicolumn{3}{c|}{skin\_nonskin}                     & \multicolumn{3}{c}{cod-rna}                           \\ \hline
		FP              & $10\%$           & $20\%$           & $30\%$           & $10\%$           & $20\%$           & $30\%$           & $10\%$           & $20\%$           & $30\%$           & $10\%$           & $20\%$           & $30\%$           \\ \hline
		\textbf{OPAUC}           & 0.958            & 0.933            & 0.845            & 0.824            & 0.804            & 0.778            & 0.925            & 0.884            & 0.815            & 0.902            & 0.864            & 0.783            \\
		\textbf{KOIL}$_{FIFO++}$ & 0.901            & 0.889            & 0.804            & 0.836            & 0.806            & 0.726            & ---              & ---              & ---              & ---              & ---              & ---              \\
		\textbf{TSAM}            & 0.961            & 0.946            & 0.838            & 0.877            & 0.846            & 0.752            & 0.937            & 0.913            & 0.842            & 0.933            & 0.880            & 0.808            \\
		\textbf{PDD}$_{SG}$      & 0.964            & 0.936            & 0.855            & 0.881            & 0.849            & 0.739            & 0.940            & 0.912            & 0.852            & 0.937            & 0.873            & 0.788            \\
		\textbf{DSAM}            & 0.983            & 0.962            & 0.862            & 0.886            & 0.837            & 0.811            & 0.961            & 0.917            & 0.819            & 0.975            & 0.922            & 0.781            \\
		\textbf{BSPAUC}          & $\textbf{0.991}$ & $\textbf{0.985}$ & $\textbf{0.945}$ & $\textbf{0.914}$ & $\textbf{0.894}$ & $\textbf{0.883}$ & $\textbf{0.979}$ & $\textbf{0.944}$ & $\textbf{0.912}$ & $\textbf{0.990}$ & $\textbf{0.956}$ & $\textbf{0.874}$ \\ \hline
		PP              & $10\%$           & $20\%$           & $30\%$           & $10\%$           & $20\%$           & $30\%$           & $10\%$           & $20\%$           & $30\%$           & $10\%$           & $20\%$           & $30\%$           \\ \hline
		\textbf{OPAUC}           & 0.923            & 0.863            & 0.803            & 0.833            & 0.816            & 0.797            & 0.927            & 0.880            & 0.856            & 0.914            & 0.893            & 0.858            \\
		\textbf{KOIL}$_{FIFO++}$ & 0.891            & 0.838            & 0.793            & 0.831            & 0.823            & 0.806            & ---              & ---              & ---              & ---              & ---              & ---              \\
		\textbf{TSAM}            & 0.930            & 0.859            & 0.809            & 0.872            & 0.849            & 0.838            & 0.933            & 0.911            & 0.877            & 0.953            & 0.903            & 0.886            \\
		\textbf{PDD}$_{SG}$      & 0.934            & 0.918            & 0.828            & 0.885            & 0.873            & 0.839            & 0.935            & 0.927            & 0.898            & 0.972            & 0.941            & 0.881            \\
		\textbf{DSAM}            & 0.902            & 0.845            & 0.757            & 0.881            & 0.852            & 0.843            & 0.980            & 0.954            & 0.902            & 0.973            & 0.938            & 0.913            \\
		\textbf{BSPAUC}          & $\textbf{0.955}$ & $\textbf{0.937}$ & $\textbf{0.876}$ & $\textbf{0.911}$ & $\textbf{0.907}$ & $\textbf{0.896}$ & $\textbf{0.995}$ & $\textbf{0.982}$ & $\textbf{0.965}$ & $\textbf{0.991}$ & $\textbf{0.973}$ & $\textbf{0.953}$ \\ \hline
	\end{tabular}
	\caption{Mean AUC results on noisy datasets. The corresponding standard deviations can be found in Appendix. (FP means the proportion of noise samples constructed by flipping labels, PP denotes the proportion of injected poison samples and  '–' means out of memory.)} \label{AUCnoise}
\end{table*}

\begin{figure*}[htb]
	\centering
	\captionsetup[subfigure]{aboveskip=1pt,belowskip=-2.3pt}
	\subfigure[rcv1]{
		\centering
		\includegraphics[width=1.6in]{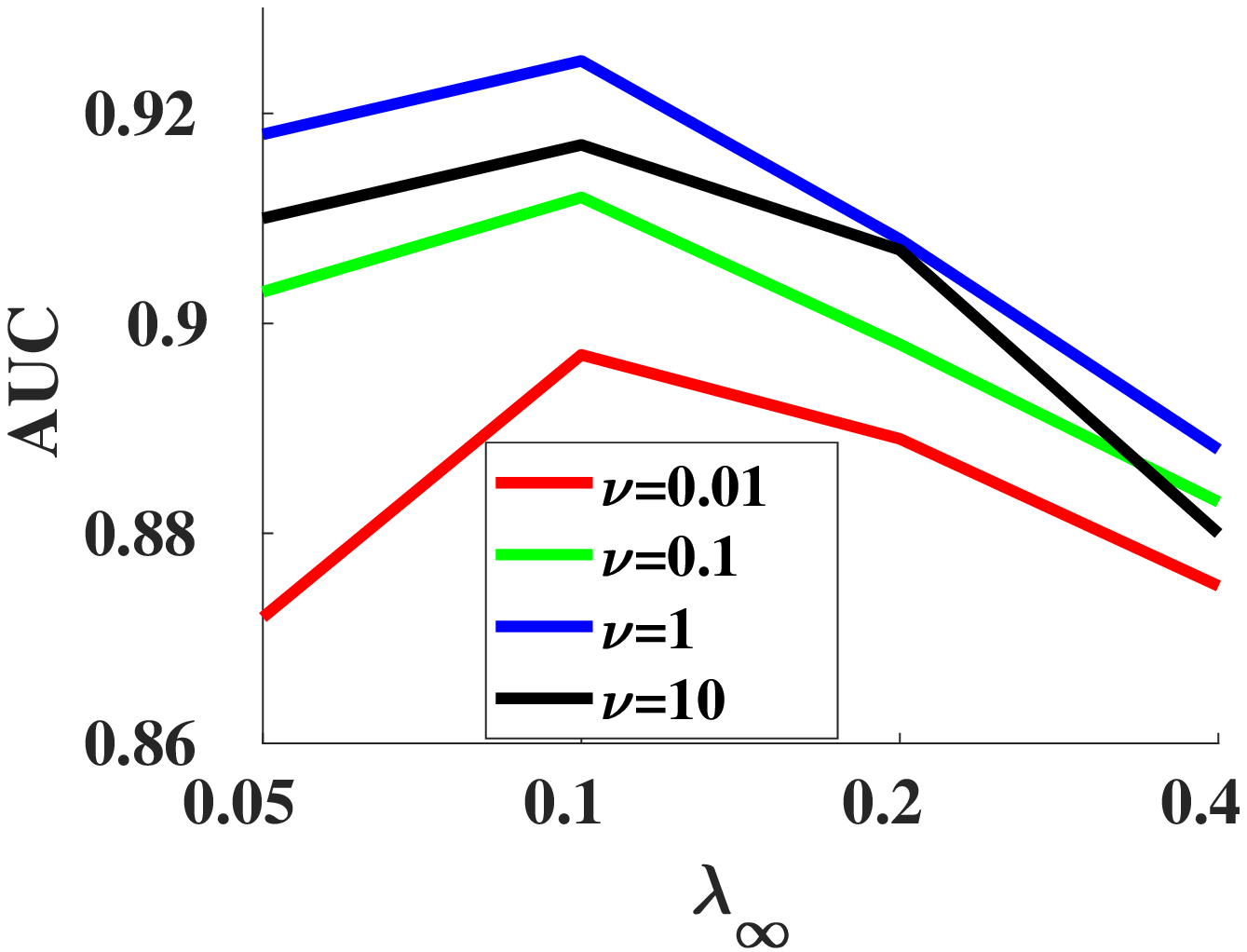}
	}
	\subfigure[a9a]{
		\centering
		\includegraphics[width=1.6in]{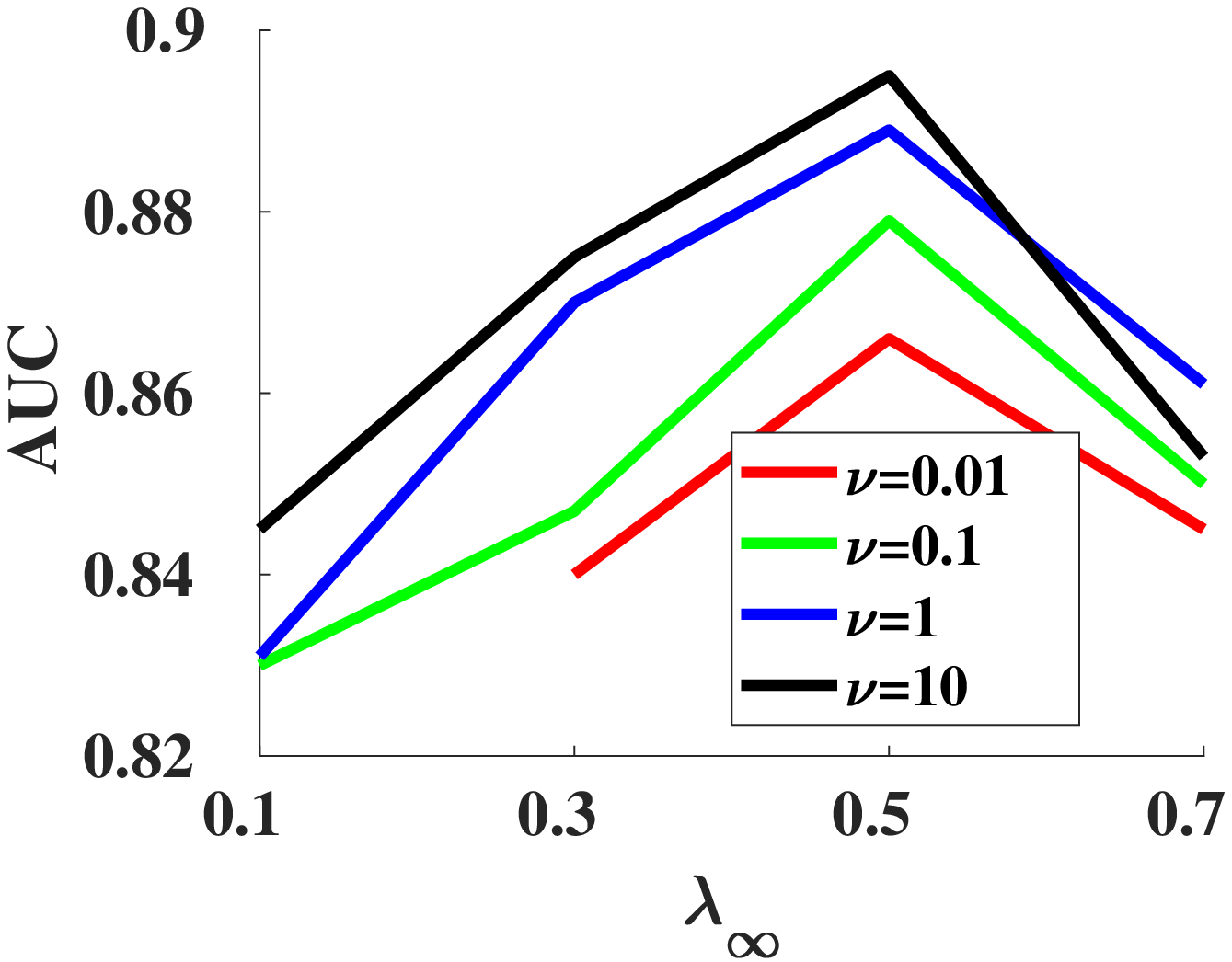}
	}
	\subfigure[skin\_nonskin]{
		\centering
		\includegraphics[width=1.6in]{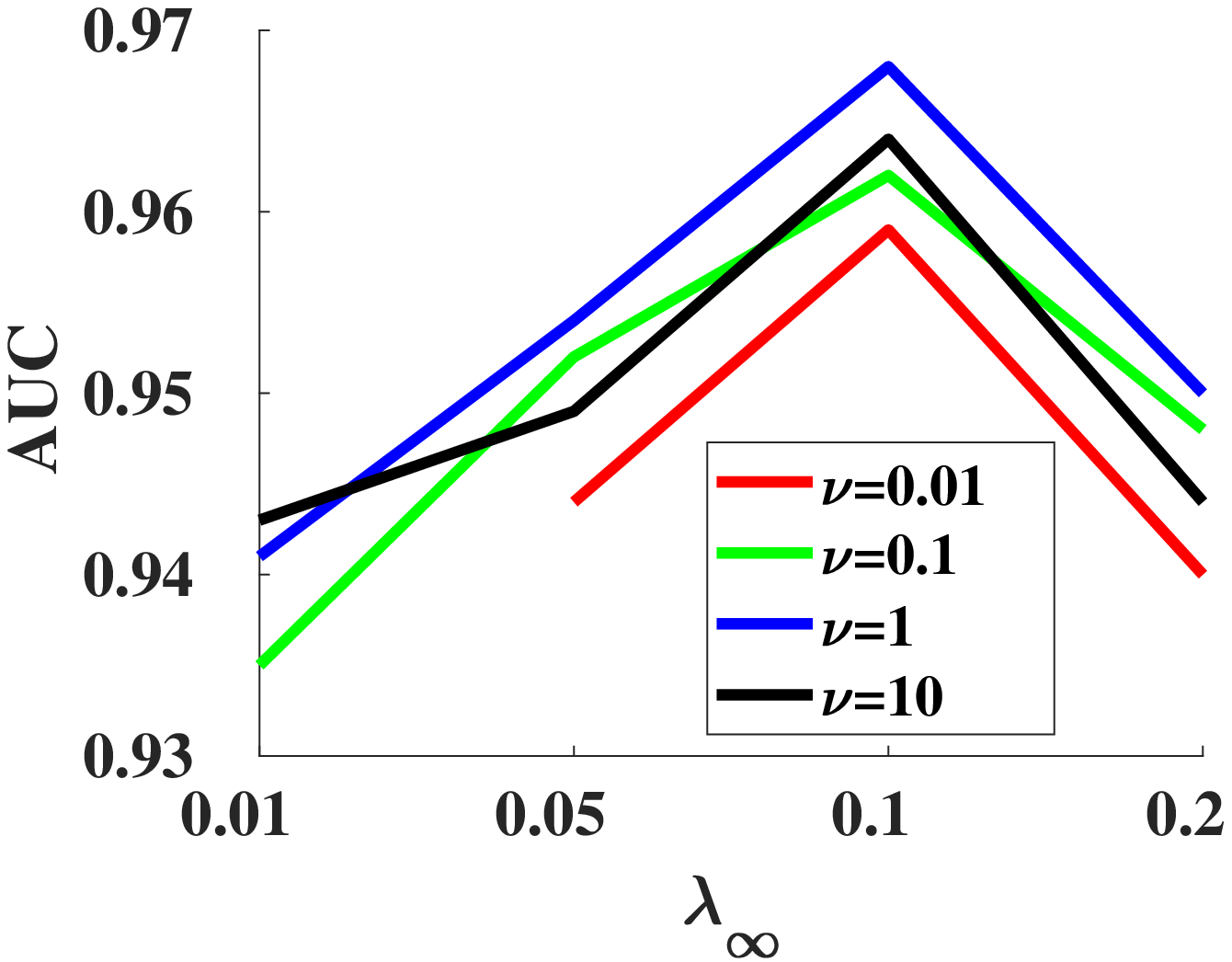}
	}
	\subfigure[cod-rna]{
		\centering
		\includegraphics[width=1.6in]{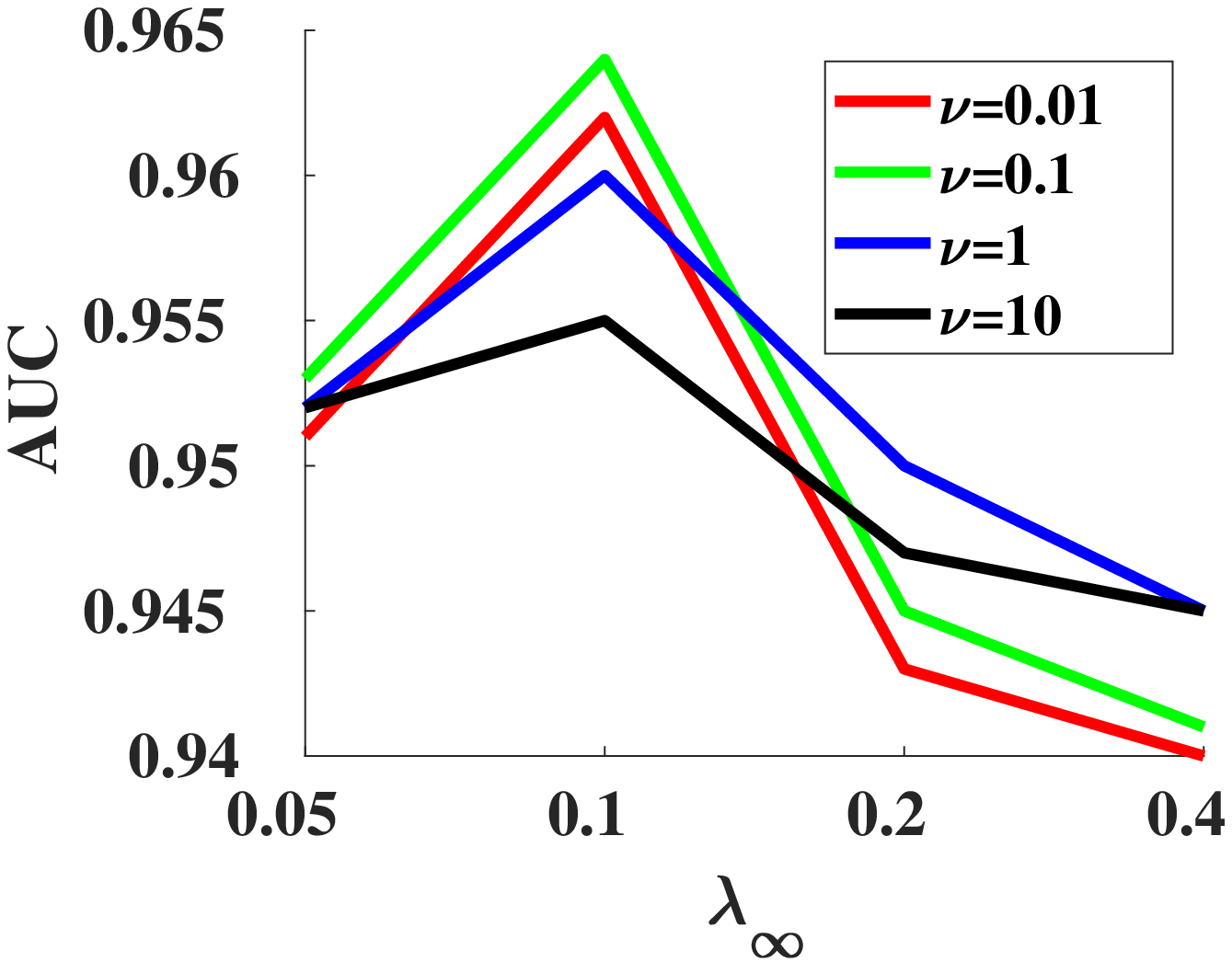}
	}
	\centering
	\caption{AUC results  with different values of $\nu$ and $\lambda_{\infty}$  on datasets with 20\% injected poison samples. (Missing results are due to only positive or negative samples are selected.)} \label{AUCParameter}
	
\end{figure*}

\begin{figure*}[htb]
	\centering
	\captionsetup[subfigure]{aboveskip=1pt,belowskip=-2.3pt}
	\subfigure[rcv1]{
		\centering
		\includegraphics[width=1.6in]{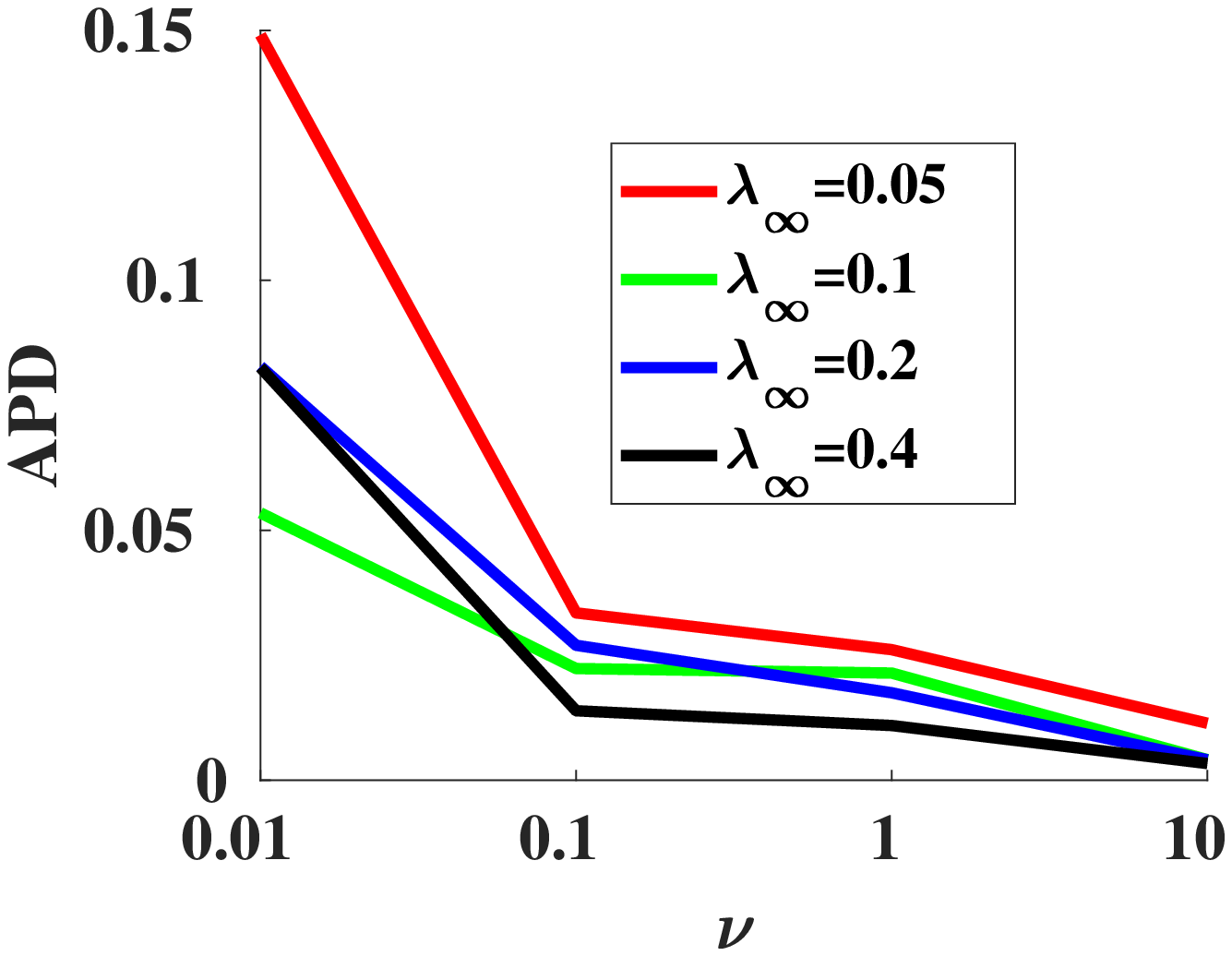}
	}
	\subfigure[a9a]{
		\centering
		\includegraphics[width=1.6in]{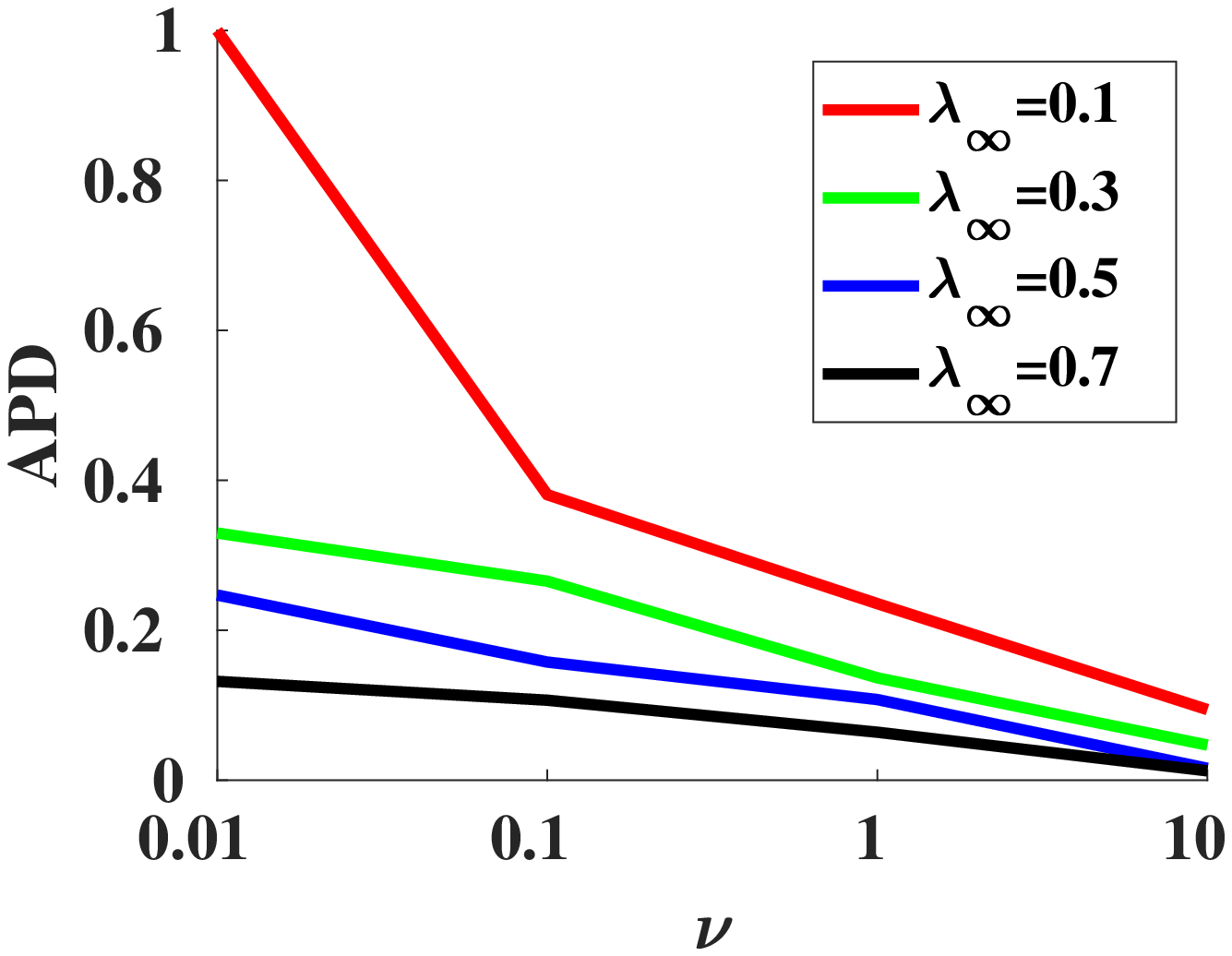}
	}
	\subfigure[skin\_nonskin]{
		\centering
		\includegraphics[width=1.6in]{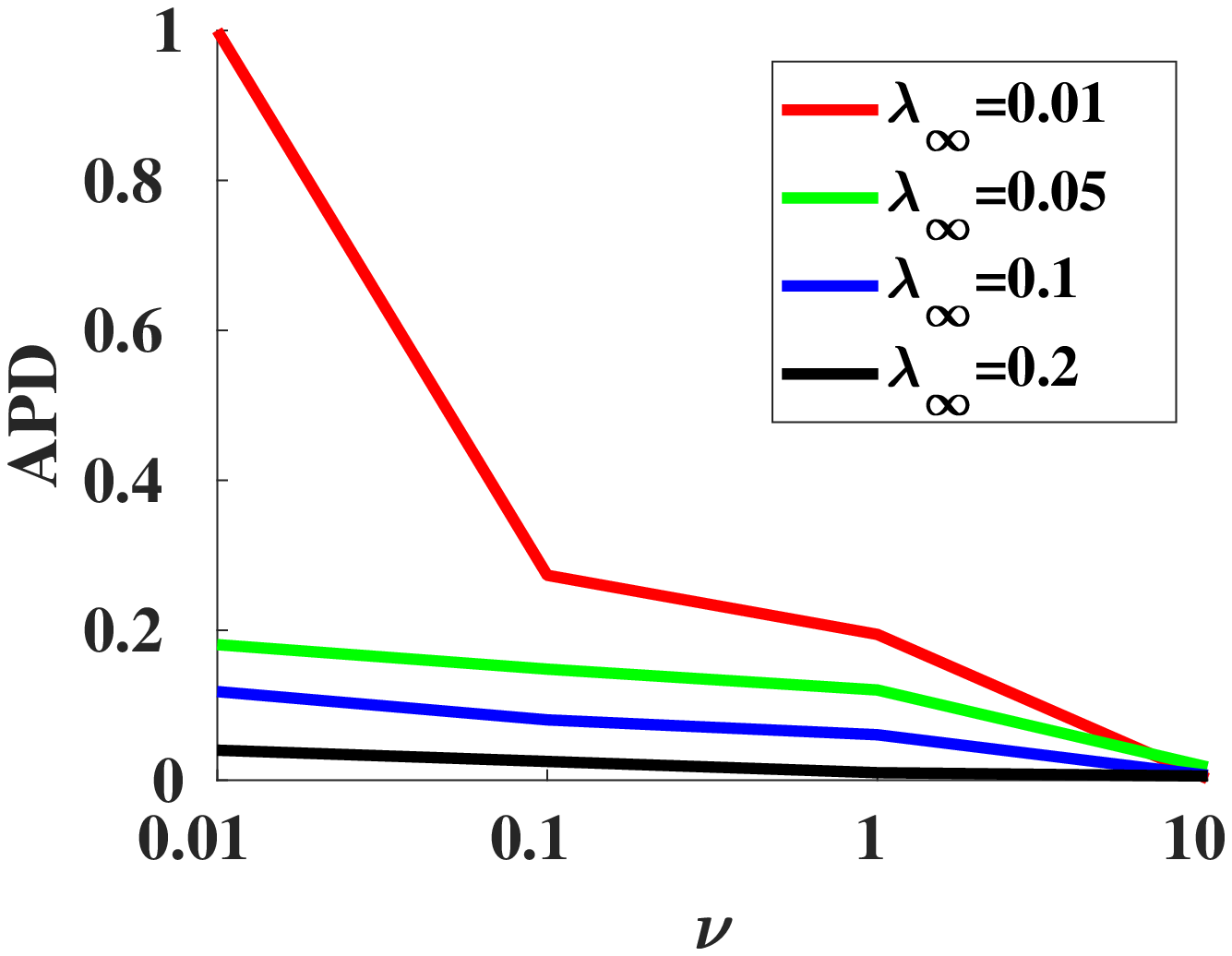}
	}
	\subfigure[cod-rna]{
		\centering
		\includegraphics[width=1.6in]{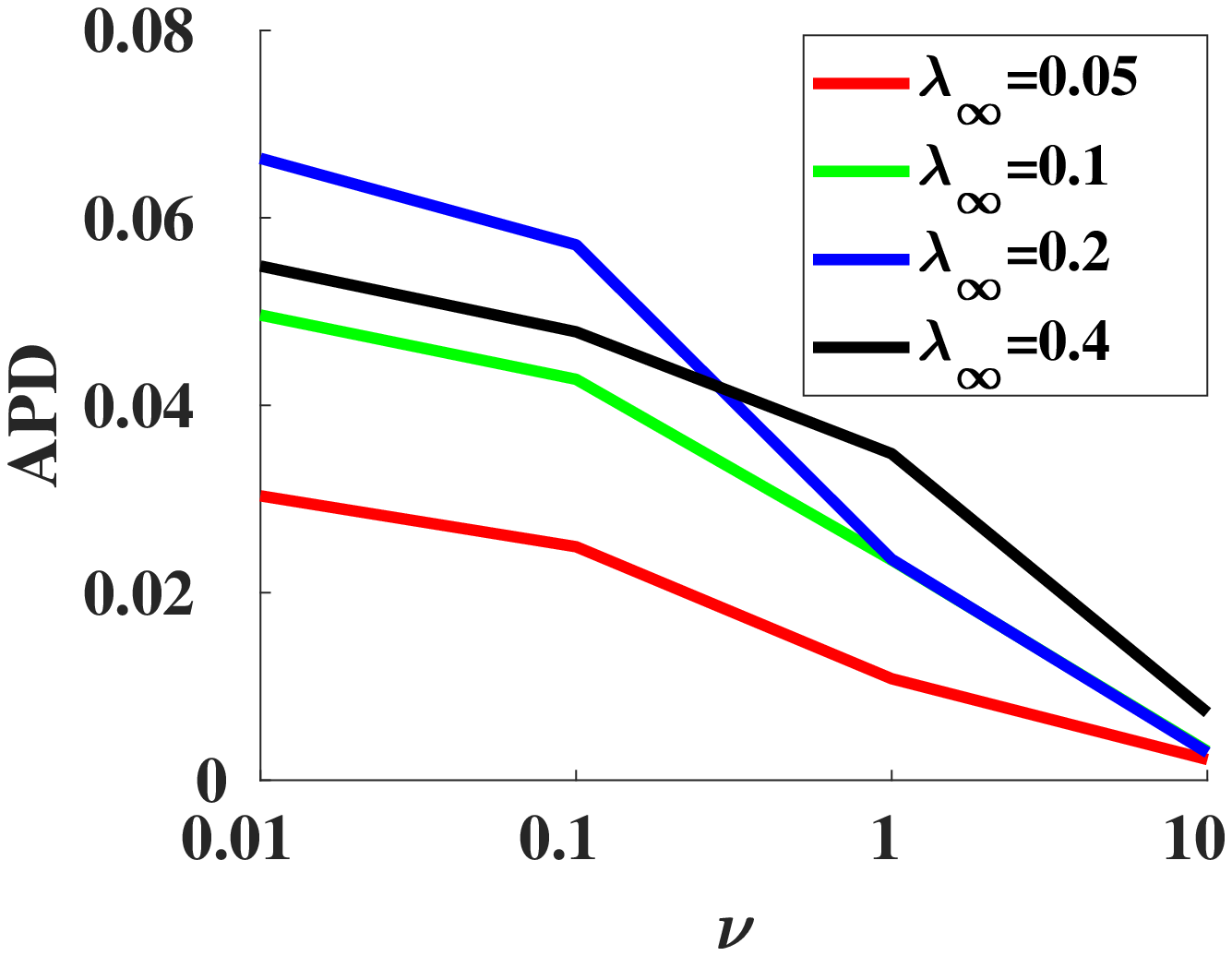}
	}
	\centering
	\caption{The results of absolute proportion  difference (APD)  with different values of $\nu$ and $\lambda_{\infty}$ on datasets with 20\% injected poison samples.} \label{APDParameter}
\end{figure*}

\subsection{Experimental Setup} \label{subsec_setup}
\noindent \textbf{Design of Experiments:} \
To demonstrate the advantage of  our  BSPAUC for handling noisy data, we compare our BSPAUC with some  state-of-the-art AUC maximization methods  on a variety of  benchmark datasets with/without artificial noisy data.  
Specifically, the compared algorithms are summarized as follows.
\\
$\mathbf{TSAM }$: A kernel-based algorithm which  updates the solution  based on  triply stochastic gradient descents \emph{w.r.t.}  random pairwise loss and random features   \cite{TSAM}. \\
$\mathbf{DSAM }$: A  modified deep learning algorithm which  updates the solution based on the doubly stochastic gradient descents \emph{w.r.t.} random pairwise loss    \cite{gu2019scalable}. \\
\textbf{KOIL$_{FIFO++}$} \footnote{KOIL$_{FIFO++}$  is available at \url{https://github.com/JunjieHu}.}: A kernelized online imbalanced learning algorithm which directly
maximizes the AUC objective with fixed budgets  for  positive   and  negative  class \cite{KOIL}.\\
$\mathbf{PPD}_{SG}$ \footnote{PPD$_{SG}$  is available at \url{https://github.com/yzhuoning}.}: A deep learning algorithm that builds on the saddle point reformulation and explores Polyak-\L ojasiewicz condition  \cite{PPD}.\\
$\mathbf{OPAUC}$ \footnote{OPAUC  is available at \url{http://lamda.nju.edu.cn/files}.} : A  linear method based on a regression algorithm, which only needs to maintain the first and second order statistics of data in memory  \cite{OPAUC}.


In addition, considering hyper-parameters $\mu$ and $\lambda_{\infty}$, we also design experiments to analyze their roles. Note that we introduce a new variable $\nu$ to illustrate the value of $\mu$, \emph{i.e.}, $\mu=\nu \lambda_{\infty}$, and a new indicator called absolute proportion  difference (APD):
$\text{APD}= \left |\frac{1}{n}\sum_{i=1}^n v_i - \frac{1}{m}\sum_{j=1}^m u_j \right |.$

\noindent \textbf{Datasets:} \ 
The benchmark datasets  are obtained from the LIBSVM repository\footnote{Datasets are available at  \url{https://www.csie.ntu.edu.tw/~cjlin/ libsvmtools/datasets}.} which take into account different dimensions and  imbalance ratios, as summarized in Table \ref{Datasets}.   The features  have been scaled to $[-1,1]$ for all datasets, and the multiclass classification datasets have been transformed to class imbalanced binary classification datasets. Specifically, we denote one class as the positive class, and the remaining classes  as the negative class. We randomly partition each dataset into $75\%$ for training and $25\%$ for testing.

In order to test the robustness of all methods, we construct two types of artificial noisy datasets.  The first method is to turn normal samples into noise samples by  flipping their labels \cite{frenay2013classification,ghosh2017robust}.  Specifically, we first utilize training set to obtain a discriminant hyperplane, and then stochastically select   samples far away from the discriminant hyperplane to flip their labels. 
Another method is to inject poison samples. 
 Specifically, we generate poison samples   for each dataset according to 
the poisoning attack method  \footnote{The poisoning attack method  is available at \url{https://github.com/ Trusted-AI/adversarial-robustness-toolbox}.} \cite{poison}, and inject these  poison samples into training set to form a noisy dataset. 
We  conduct experiments with different noise proportions (from 10\% to 30\%).

\noindent \textbf{Implementation:} \ 
All the experiments are conducted  on a PC with 48 2.2GHz cores, 80GB RAM and 4 Nvidia 1080ti GPUs and  all the results are the average of 10 trials. We complete the deep learning and kernel-based implementations  of our BSPAUC  in python, we also implement the TSAM and the modified DSAM  in python. We use the open codes as the implementations of KOIL$_{FIFO++}$, PPD$_{SG}$ and OPAUC, which are provided by their authors. 

For the OPAUC algorithm on high dimensional datasets (the feature size is larger than $1000$), we use the low-rank version, and set the rank parameter at $50$.  For all kernel-based methods, we use Gaussian kernel $k(x,x')=\exp{( - \frac{ ||x-x'||^2}{2\sigma^2}) }$ and tune its hyper-parameter $\sigma \in 2^{[-5,5]}$ by a 5-fold cross-validation. For TSAM method and our Algorithm 3 (in Appendix), the number of random Fourier features is selected from $[500 : 500 : 4000]$. For KOIL$_{FIFO++}$ method, the buffer  size  is set at  $100$  for  each  class.   For all deep learning methods, we utilize the same network structure which consists of  eight full connection layers and uses the ReLu activation function. 
For PPD$_{SG}$ method, the initial stage is tuned from $200$ to $2000$. For our BSPAUC, the hyper-parameters are chosen according to the proportion of selected samples. Specifically,  we start training  with  about $50\%$ samples, and then linearly increase $\lambda$ to include more samples. Note that all algorithms have the same  pre-training that we  select a small number of samples for training and take the trained model as the initial state of experiments.

\subsection{Results and Discussion}

First of all, we explain the missing part of the experimental results on Tables \ref{AUCR} and \ref{AUCnoise}. The KOIL$_{FIFO++}$  is a kernel-based method and needs to calculate and save the kernel matrix. For large datasets, such operation could cause out of memory. The PPD$_{SG}$ does not consider the severe imbalance of datasets and it can only produce negligible updates when the stochastic  batch samples are severely  imbalanced. If facing these issues, these two algorithms could not work as usual. Thus,  we cannot provide the corresponding results.

Table \ref{AUCR} presents the mean AUC results with the corresponding standard deviation of all  algorithms on the original benchmark datasets. The results show that, due to the  SPL used in our BSPAUC, the deep learning and kernel-based implementations of our  BSPAUC outperforms the DSAM and TSAM, which are the non self-paced versions of our implementations.  At the same time, our  BSPAUC also obtains better AUC results compared to other existing state-of-the-art  AUC maximization methods (\emph{i.e.}, OPAUC and PPD$_{SG}$).

Table \ref{AUCnoise}  shows the performance of the  deep learning implementation of our BSPAUC and other compared methods  on the two types of noisy datasets. The results clearly show that our BSPAUC achieves the best performance on all noisy datasets. Specifically, the larger the proportion of noisy data is, the more obvious the advantage is.   Our BSPAUC excludes the noise samples from training   by giving a zero weight  and  thus has a better robustness.

Figures \ref{AUCParameter} and \ref{APDParameter}  show the results of the  deep learning implementation of BSPAUC with different hyper-parameter values. Figure \ref{AUCParameter} clearly reveals that with the increase of $\lambda_{\infty}$, AUC  increases first and then decreases gradually. This phenomenon is expected. When $\lambda_{\infty}$ is small, increasing $\lambda_{\infty}$ 
will cause more easy samples to join the  training  and thus the generalization of  model is improved. However, when $\lambda_{\infty}$ is large enough,  complex (noise) samples start being included and then AUC  decreases. What's more, Figure \ref{APDParameter}  directly reflects that with the increase of $\mu$, which can be calculated by $\mu=\nu \lambda_{\infty}$, the proportions of selected positive and negative samples gradually approach. 
Further more, combining with the above two figures, we observe that too large APD often tends to low AUC. Large APD often implies that  some easy samples in the class with low  proportion of selected samples don't join the training while some complex samples in the other class with high proportion are selected. The above case leads to the reduction of generalization ability and thus causes low AUC. Importantly, our balanced self-paced regularization term is proposed for this issue.

\section{Conclusion}
In this paper,  we first provide a statistical explanation to self-paced AUC. Inspired by this, we propose our self-paced  AUC maximization formulation with a novel balanced self-paced regularization term. Then we propose a doubly cyclic block coordinate descent algorithm (\emph{i.e.}, BSPAUC) to optimize our  objective function. 
Importantly, we prove that the sub-problem with respect to all weight variables converges to a stationary point on the basis of closed-form solutions, and  our BSPAUC converges to a stationary point of our fixed optimization   objective under a mild assumption.  The experimental results demonstrate that our BSPAUC outperforms existing state-of-the-art  AUC maximization methods and  has better  robustness.
\section{Acknowledgments}
Bin Gu was partially supported by the National Natural Science Foundation of China (No:61573191).

\section{Appendix A. Table \ref{AUCR} with standard deviation}
Mean AUC results with the corresponding standard deviation on noisy datasets are shown in Table \ref{AUCR}. The results clearly show that our BSPAUC achieves the best performance on all noisy datasets. Specifically, the larger the proportion of noisy data is, the more obvious the advantage is.   Our BSPAUC excludes the noise samples from training   by giving a zero weight  and  thus has a better robustness.
\setcounter{table}{2}
\begin{table*} [htb]   \small
	\centering
	\caption{\small{Mean AUC results with the corresponding standard deviation on noisy datasets. (FP means the proportion of noise samples constructed by flipping labels, PP denotes the proportion of injected poison samples and  '–' means out of memory.)}}  \label{AUCR}
	\begin{tabular}{c|ccc|ccc|}
		\hline
		Datasets                 & \multicolumn{3}{c|}{rcv1}                                                            & \multicolumn{3}{c|}{a9a}                                                             \\ \hline
		FP                       & \multicolumn{1}{c|}{$10\%$} & \multicolumn{1}{c|}{$20\%$} & $30\%$                   & \multicolumn{1}{c|}{$10\%$} & \multicolumn{1}{c|}{$20\%$} & $30\%$                   \\ \hline
		\textbf{OPAUC}           & 0.958$\pm$0.015             & 0.933$\pm$0.025             & 0.845$\pm$0.023          & 0.824$\pm$0.009             & 0.804$\pm$0.013             & 0.778$\pm$0.016          \\
		\textbf{KOIL}$_{FIFO++}$ & 0.901$\pm$0.026             & 0.889$\pm$0.031             & 0.804$\pm$0.042          & 0.836$\pm$0.015             & 0.806$\pm$0.024             & 0.726$\pm$0.033          \\
		\textbf{TSAM}            & 0.961$\pm$0.002             & 0.946$\pm$0.015             & 0.838$\pm$0.024          & 0.877$\pm$0.012             & 0.846$\pm$0.015             & 0.752$\pm$0.018          \\
		\textbf{PDD}$_{SG}$      & 0.964$\pm$0.002             & 0.936$\pm$0.007             & 0.855$\pm$0.017          & 0.881$\pm$0.002             & 0.849$\pm$0.002             & 0.739$\pm$0.006          \\
		\textbf{DSAM}            & 0.983$\pm$0.005             & 0.962$\pm$0.011             & 0.862$\pm$0.034          & 0.886$\pm$0.003             & 0.837$\pm$0.007             & 0.811$\pm$0.006          \\
		\textbf{BSPAUC}          & \textbf{0.991$\pm$0.002}    & \textbf{0.985$\pm$0.009}    & \textbf{0.945$\pm$0.009} & \textbf{0.914$\pm$0.006}    & \textbf{0.894$\pm$0.012}    & \textbf{0.883$\pm$0.010} \\ \hline
		Datasets                 & \multicolumn{3}{c|}{skin\_nonskin}                                                   & \multicolumn{3}{c|}{cod-rna}                                                         \\ \hline
		FP                       & \multicolumn{1}{c|}{$10\%$} & \multicolumn{1}{c|}{$20\%$} & $30\%$                   & \multicolumn{1}{c|}{$10\%$} & \multicolumn{1}{c|}{$20\%$} & $30\%$                   \\ \hline
		\textbf{OPAUC}           & 0.925$\pm$0.009             & 0.884$\pm$0.009             & 0.815$\pm$0.008          & 0.902$\pm$0.021             & 0.864$\pm$0.036             & 0.783$\pm$0.034          \\
		\textbf{KOIL}$_{FIFO++}$ & ---                         & ---                         & ---                      & ---                         & ---                         & ---                      \\
		\textbf{TSAM}            & 0.937$\pm$0.001             & 0.913$\pm$0.004             & 0.842$\pm$0.007          & 0.933$\pm$0.006             & 0.880$\pm$0.004             & 0.808$\pm$0.015          \\
		\textbf{PDD}$_{SG}$      & 0.940$\pm$0.001             & 0.912$\pm$0.004             & 0.852$\pm$0.012          & 0.937$\pm$0.002             & 0.873$\pm$0.006             & 0.788$\pm$0.021          \\
		\textbf{DSAM}            & 0.961$\pm$0.003             & 0.917$\pm$0.002             & 0.819$\pm$0.014          & 0.975$\pm$0.002             & 0.922$\pm$0.002             & 0.781$\pm$0.016          \\
		\textbf{BSPAUC}          & \textbf{0.979$\pm$0.001}    & \textbf{0.944$\pm$0.002}    & \textbf{0.912$\pm$0.010} & \textbf{0.990$\pm$0.001}    & \textbf{0.956$\pm$0.001}    & \textbf{0.874$\pm$0.007} \\ \hline
		Datasets                 & \multicolumn{3}{c|}{rcv1}                                                            & \multicolumn{3}{c|}{a9a}                                                             \\ \hline
		PP                       & \multicolumn{1}{c|}{$10\%$} & \multicolumn{1}{c|}{$20\%$} & $30\%$                   & \multicolumn{1}{c|}{$10\%$} & \multicolumn{1}{c|}{$20\%$} & $30\%$                   \\ \hline
		\textbf{OPAUC}           & 0.923$\pm$0.016             & 0.863$\pm$0.020             & 0.803$\pm$0.026          & 0.833$\pm$0.015             & 0.816$\pm$0.018             & 0.797$\pm$0.016          \\
		\textbf{KOIL}$_{FIFO++}$ & 0.891$\pm$0.023             & 0.838$\pm$0.027             & 0.793$\pm$0.044          & 0.831$\pm$0.022             & 0.823$\pm$0.034             & 0.806$\pm$0.035          \\
		\textbf{TSAM}            & 0.930$\pm$0.002             & 0.859$\pm$0.012             & 0.809$\pm$0.021          & 0.872$\pm$0.002             & 0.849$\pm$0.002             & 0.838$\pm$0.005          \\
		\textbf{PDD}$_{SG}$      & 0.934$\pm$0.003             & 0.918$\pm$0.005             & 0.828$\pm$0.017          & 0.885$\pm$0.002             & 0.873$\pm$0.001             & 0.839$\pm$0.004          \\
		\textbf{DSAM}            & 0.902$\pm$0.007             & 0.845$\pm$0.017             & 0.757$\pm$0.051          & 0.881$\pm$0.004             & 0.852$\pm$0.004             & 0.843$\pm$0.009          \\
		\textbf{BSPAUC}          & \textbf{0.955$\pm$0.002}    & \textbf{0.937$\pm$0.003}    & \textbf{0.876$\pm$0.010} & \textbf{0.911$\pm$0.002}    & \textbf{0.907$\pm$0.002}    & \textbf{0.896$\pm$0.008} \\ \hline
		Datasets                 & \multicolumn{3}{c|}{skin\_nonskin}                                                   & \multicolumn{3}{c|}{cod-rna}                                                         \\ \hline
		PP                       & \multicolumn{1}{c|}{$10\%$} & \multicolumn{1}{c|}{$20\%$} & $30\%$                   & \multicolumn{1}{c|}{$10\%$} & \multicolumn{1}{c|}{$20\%$} & $30\%$                   \\ \hline
		\textbf{OPAUC}           & 0.927$\pm$0.010             & 0.880$\pm$0.020             & 0.856$\pm$0.021          & 0.914$\pm$0.015             & 0.893$\pm$0.023             & 0.858$\pm$0.034          \\
		\textbf{KOIL}$_{FIFO++}$ & ---                         & ---                         & ---                      & ---                         & ---                         & ---                      \\
		\textbf{TSAM}            & 0.933$\pm$0.002             & 0.911$\pm$0.004             & 0.877$\pm$0.005          & 0.953$\pm$0.002             & 0.903$\pm$0.006             & 0.886$\pm$0.005          \\
		\textbf{PDD}$_{SG}$      & 0.935$\pm$0.001             & 0.927$\pm$0.002             & 0.898$\pm$0.009          & 0.972$\pm$0.001             & 0.941$\pm$0.001             & 0.881$\pm$0.014          \\
		\textbf{DSAM}            & 0.980$\pm$0.001             & 0.954$\pm$0.001             & 0.902$\pm$0.005          & 0.973$\pm$0.004             & 0.938$\pm$0.005             & 0.913$\pm$0.004          \\
		\textbf{BSPAUC}          & \textbf{0.995$\pm$0.001}             &\textbf{0.982$\pm$0.001}            & \textbf{0.965$\pm$0.006} & \textbf{0.991$\pm$0.003}    & \textbf{0.973$\pm$0.003}    & \textbf{0.953$\pm$0.008} \\ \hline
	\end{tabular}
\end{table*}
\section{Appendix B. Implementation of Algorithm \ref{Kernel}}
When we fix $\mathbf{v},\mathbf{u}$ to update $\theta$:
\begin{equation} 
	\label{f}
	\theta^{t}= \argmin \limits_{ \theta}  \frac{1}{nm} \sum_{i=1}^{n}\sum_{j=1}^{m}v_i u_j \xi_{ij}
	+\tau \Omega(\theta)  + \text{const}, 
\end{equation}
where $\xi_{ij}=\max \{ 1-f(x^+_i)+f(x^-_j),0 \} $. 

For the kernel-based implementation of solving Eq. (\ref{f}),  we  also compute the gradient on random pairs of weighted  samples which  are selected  from two subsets  $\hat{X}^+$ and $\hat{X}^-$  respectively.  In this case, the weighted batch AUC loss is defined as:
\begin{small}
	\begin{align} \label{DF}
		\sum_{j=1}^\pi v_j u_j  \xi_{jj} =\sum_{j=1}^\pi v_j u_j  \max \{1-f(x^+_j)+f(x^-_j), 0 \}.
\end{align}\end{small}
Then, we apply random Fourier feature method to approximate the kernel function \cite{rahimi2008random,dai2014scalable} for  large-scale problems. The mapping function of $D$ random features is defined as
\begin{small}
	\begin{align*}
		\phi_{\omega}(x)=\sqrt{1/D}[ &\cos(\omega_1x),\ldots,\cos(\omega_Dx),  \\ &\sin(\omega_1x),\ldots,\sin(\omega_Dx)]^T, 
\end{align*}\end{small}where  $\omega_i$ is randomly sampled according to the density function $p(\omega)$ associated with $k(x,x')$ \cite{odland2017fourier}.  Then, based on the  weighted batch AUC loss (\ref{DF}) and the random feature mapping function $\phi_{\omega}(x)$, we obtain Algorithm \ref{Kernel}  by applying the  triply stochastic gradient descent  method (TSGD) \cite{TSAM} \emph{w.r.t.}  random  pairs of weighted samples and random features.
\setcounter{algorithm}{2}
\begin{algorithm} [H]
	\caption{Kernel-based implementation of solving Eq.   (\ref{f})}  \label{Kernel}
	\begin{algorithmic}[1]   
		\REQUIRE $p(\omega),\phi_{\omega}(x),\tau, T, \mathbf{\eta}, \hat{X}^+, \hat{X}^-, \pi,\theta^0$.
		\FOR{$i=1, \cdots ,T$}
		\STATE Sample $\hat{x}^+_1,...,\hat{x}^+_\pi$ from $\hat{X}^+$.
		\STATE Sample $\hat{x}^-_1,...,\hat{x}^-_\pi$ from $\hat{X}^-$.
		\STATE Sample $\omega_i$ $\sim$ $p(\omega)$ with seed $i$.
		\STATE Calculate $f({x})$ by \textbf{Predict}$(x,\{\alpha_i\}_{j=1}^{i-1})$.
		\STATE Get $\alpha_i$ according to the following formula:
		\begin{align*}  \small
			\alpha_i=&-\frac{\eta_i}{\pi}\sum_{j=1}^\pi v_j u_j \frac{\partial \xi_{jj} }{ \partial \theta } .\qquad \qquad \qquad \qquad
		\end{align*}
		\STATE Update $\alpha_j=(1-\eta_j\tau)\alpha_j,j\in [1,i-1]$.
		\ENDFOR
		\ENSURE $\theta=\{\alpha_i\}_{i=1}^T$.
	\end{algorithmic}
\end{algorithm}
\begin{algorithm} [H]
	\caption{$f(x)=$ \textbf{Predict}$(x,\{\alpha_i\}_{i=1}^{t})$}
	\begin{algorithmic}[1]
		\REQUIRE $p(\omega),\phi_{\omega}(x),x,\{\alpha_i\}_{i=1}^{t}$.
		\STATE Set $f(x)=0$.
		\FOR { $i=1, \cdots ,t$}
		\STATE Sample $\omega_i$ $\sim$ $p(\omega)$ with seed $i$.
		\STATE $f(x)=f(x)+\alpha_i \phi_{\omega_i}(x)$.
		\ENDFOR
		\ENSURE $f(x)$.
	\end{algorithmic}
\end{algorithm}
\section{Appendix C. Proof of Theorem \ref{theoremUp}}
Firstly, we introduce some notions again.
Let $X$ be a compact subset of $\mathbb{R}^d$, $Y = \{-1, +1\}$ be the label set and $Z = X \times Y$. Given a distribution $P(z)$ and let  $S=\{z_i=(x_i,y_i)\}_{i=1}^n$ be an independent and identically distributed (IID)  training set drawn from $P(z)$, where $x_i \in X$, $y_i \in Y$ and $z_i \in Z$. Thus, empirical  AUC risk on $S$ can be formulated as:
\begin{align} \label{emp}
	R_{emp}(S;f)=\frac{1}{n(n-1)}\sum_{z_i,z_j \in S, z_i \neq z_j} L_f(z_i,z_j).
\end{align}
Here, $f \in \mathcal{F}: \mathbb{R}^d \to \mathbb{R}$ is one real-valued function and the pairwise loss function $L_f(z_i,z_j)$ for AUC is defined as:
\begin{equation*} \small
	L_f(z_i,z_j)=
	\left \{\begin{array} {l@{\ \ \textrm{if}  \ \ }l} 0   &  \ y_i=y_j
		\\ \mathbb{I}(f(x_i) \le f(x_j))  &y_i =+1 \& \ y_j=-1\\
		\mathbb{I}(f(x_j) \le f(x_i))  &y_j =+1 \& \ y_i=-1
	\end{array} \right.
\end{equation*}
where $\mathbb{I}(\cdot)$ is the indicator function such that
$\mathbb{I}(\pi)$ equals 1 if $\pi$ is true and 0 otherwise.

Further, the expected risk of AUC for the  distribution $P(z)$ can be defined as:
\begin{align}\label{exp}
	&R_{exp}(P(z);f):=\mathbb{E}_{z_1,z_2 \sim P(z)^2 } L_f(z_1,z_2) \nonumber  \\
	=&\mathbb{E}_{S \sim P(z)^n } \left[ \frac{1}{n(n-1)}  \sum_{z_i,z_j \in S, z_i \neq z_j}  L_f(z_i,z_j) \right]   \nonumber\\
	=&\mathbb{E}_{S}[R_{emp}(S;f)].
\end{align}

Under the assumption that  our training set in reality is composed by not only target clean data but also  a  proportion of noise samples, Gong et al. \cite{why} connect the data distribution $P_{train}(z)$ of the training set  with the distribution $P_{target}(z)$ of the target set using a  weight function $W_{\lambda}(z)$:
\begin{align} \small \label{OriginalF} 
	& P_{target}(z)=\frac{1}{\alpha_*}W_{\lambda}(z)P_{train}(z),  \\
	&\alpha_*=\int_{Z}W_{\lambda}(z)P_{train}(z)dz , \nonumber
\end{align}
where  $0 \le W_{\lambda}(z) \le 1 $ and $\alpha_*$  denotes the normalization factor.  Intuitively, $W_{\lambda}(z)$ gives larger weights to easy samples than to complex samples and with the increase of pace parameter $\lambda$, all samples tend to be assigned larger weights.  

Then, Eq. (\ref{OriginalF}) can be   reformulated as:
\begin{align} \label{NewF}
	& P_{train}(z)=\alpha_*P_{target}(z)+(1-\alpha_*)E(z), \\
	& E(z)=\frac{1}{1-\alpha_*}(1-W_{\lambda_*}(z))P_{train}(z). \nonumber
\end{align}
Based on (\ref{NewF}), we define the pace distribution $Q_{\lambda}(z)$ as:
\begin{align} \label{Q}
	Q_{\lambda}(z)=\alpha_{\lambda}P_{target}(z)+(1-\alpha_{\lambda})E(z),
\end{align}
where $\alpha_{\lambda}$ varies from $1$ to $\alpha_{*}$ with increasing pace parameter $\lambda$.  Correspondingly,  $Q_{\lambda}(z)$ simulates the changing process from $P_{target}(z)$ to $P_{train}(z)$.  Note that $Q_{\lambda}(z)$ can also be regularized into the following formulation:
\begin{align}   \label{Q2}
	Q_{\lambda}(z)  \propto  W_{\lambda}(z)P_{train}(z).
\end{align}
Here,
\begin{align*}   
	W_{\lambda}(z) \propto  \frac{\alpha_{\lambda}P_{target}(z)+(1-\alpha_{\lambda})E(z)}{\alpha_*P_{target}(z)+(1-\alpha_*)E(z)},
\end{align*}
and
\begin{align*}   
 &\frac{\alpha_{\lambda}P_{target}(z)+(1-\alpha_{\lambda})E(z)}{\alpha_*P_{target}(z)+(1-\alpha_*)E(z)}\\
 =&\frac{\alpha_{\lambda}P_{target}(z)+(1-\alpha_{\lambda})E(z)}{P_{train}(z)},
\end{align*}
where $0 \le  W_{\lambda}(z) \le 1 $ through normalizing its maximal value to 1.

Then, after  introducing the definition of McDiarmid inequality, we obtain the relationship between the empirical AUC risk  $R_{emp}$ (\ref{emp}) and the expected AUC risk $R_{exp}$ (\ref{exp}) in Lemma \ref{TwoR}.
\begin{definition} \label{Ineq}
	(McDiarmid Inequality) Let $z_1,...z_n \in S$  be a set of independent random variables and assume that there exists $c_1,...,c_n >0$ such that $f:S^n \rightarrow \mathbb{R}$ satisfies the following inequalities:
	\begin{align*}
		|f(z_1,...,z_i,...,z_n)-f(z_1,...,z_i',...z_n)| \le c_i  ,
	\end{align*}
	for all $i \in \{1,...,n\}$ and points $z_1',...z_n' \in S$. Then, $\forall \epsilon >0$ we have
	\begin{small}
	\begin{align*}
		&P(f(z_1,...,z_n)-\mathbb{E}(f(z_1,...,z_n)) \ge  \epsilon) \le \exp \left( \frac{-2 \epsilon^2}{\sum_{i=1}^n c_i^2} \right ) \\
		&P(f(z_1,...,z_n)-\mathbb{E}(f(z_1,...,z_n)) \le  -\epsilon) \le \exp \left (\frac{-2 \epsilon^2}{\sum_{i=1}^n c_i^2} \right) 
	\end{align*}
	\end{small}
\end{definition}
\begin{lemma} \label{TwoR}
	An independent and identically distributed sample set $S=\{z_1,...,z_n\}, z_i \in Z$ is obtained according to the distribution $P(z)$. Then for any $\delta >  0$  and  any $ f \in \mathcal{F}$, the following holds for $R_{emp}$ (\ref{emp}) and $R_{exp}$ (\ref{exp}) with confidence at least   $1-\delta$:
	\begin{align}  
		R_{exp}(P(z);f) \le R_{emp}(S;f) +\sqrt{\frac{\ln (1/\delta)}{n/2}}.
	\end{align}
\end{lemma}
\begin{proof}
	Let $S'$ be a set of samples same as $S$ and then we change the $z \in S'$ to $z'$. In this case,  $S'$ is different from $S$ with only one sample and $z \in S,z'\in S'$ are the two different samples. Then, we have
	\begin{align*}  
		&R_{emp}(S;f)-R_{emp}(S';f)\\
		\overset{a}{=}&\frac{2}{n(n-1)}\left(\sum_{z_j \in S, z_j \neq z} L_f(z,z_j)-\sum_{z_j \in S', z_j \neq z' } L_f(z',z_j)\right) \\
		\overset{b}{\le} &\frac{2(n-1)}{n(n-1)} = \frac{2}{n}
	\end{align*}
	The equality (a) is due to the property that  $\forall z_i,z_j \in Z, L_f(z_i,z_j)=L_f(z_j,z_i)$,  and the inequality (b) is obtained by the fact that the function $L_f$ is bounded by $[0,1]$. Similarly, we can get 
	\begin{align*}  
		|R_{emp}(S;f)-R_{emp}(S';f)| \le \frac{2}{n}.
	\end{align*}
	
	According to Definition \ref{Ineq}, $\forall \epsilon >0$ we have 
	\begin{equation*} \small
		\begin{aligned}  
			&P( R_{emp}(S;f) - E_{S}[R_{emp}(S;f)] \le - \epsilon ) \le \exp (\frac{-\epsilon^2}{2/n}) \\
			\iff & P( R_{emp}(S;f) - R_{exp}(P(z);f) \ge - \epsilon ) \ge 1- \exp (\frac{-\epsilon^2}{2/n}) 
		\end{aligned}
	\end{equation*}
	Then, we define $\delta$ as $\exp (\frac{-\epsilon^2}{2/n})$ and calculate $\epsilon$ as $\sqrt{\frac{\ln (1/\delta)}{n/2}}$.\\
	In this case, we proof that with confidence at least   $1 - \delta$, the following holds
	\begin{align}  
		R_{exp}(P(z);f) \le R_{emp}(S;f) +\sqrt{\frac{\ln (1/\delta)}{n/2}}.
	\end{align}
	The proof is then completed.
\end{proof}
Combining with the pace distribution $Q_{\lambda}$ (\ref{Q}), we get the following Lemma.
\begin{lemma} \label{old}
	An independent and identically distributed sample set $S=\{z_1,...,z_n\}$ is obtained according to the pace distribution $Q_{\lambda}$ (\ref{Q}). Then for any $\delta>0$  and  any $ f \in \mathcal{F}$, with confidence at least $1-\delta$ we have:
	\begin{align} 
		R_{exp}(P_{target};f)  & \le  R_{emp}(S;f)+e_{\lambda} +\sqrt{\frac{\ln (1/\delta)}{n/2}}
	\end{align}
	where $e_{\lambda}$ is defined as $R_{exp}(P_{target};f) - R_{exp}(Q_{\lambda};f)$ and  with the increasing of $\lambda$, $e_{\lambda}$ decreases monotonically from $0$.
\end{lemma}
\begin{proof}
	The empirical risk in training set tends not to approximate the expected risk due to the inconsistence of $P_{train}$ and $P_{target}$. However, by introducing pace empirical risk with pace distribution $Q_{\lambda}$ in the error analysis, we can formulate the following error decomposition:
	\begin{small}
	\begin{align} \label{S12} 
		&R_{exp}(P_{target};f) - R_{emp}(S;f)  \\
		=&R_{exp}(P_{target};f)-R_{exp}(Q_{\lambda};f)+R_{exp}(Q_{\lambda};f)-R_{emp}(S;f), \nonumber
	\end{align}
	\end{small}
	and we define 
	\begin{align}  \label{S1}
		e_{\lambda}=R_{exp}(P_{target};f) - R_{exp}(Q_{\lambda};f).
	\end{align}
	Considering the definition of $Q_{\lambda}$ (\ref{Q}) which simulates the changing process from $P_{target}$ to $P_{train}$   and the relationship  (\ref{OriginalF}) between $P_{train}$  and $P_{target}$, we conclude that with the increasing of $\lambda$, $e_{\lambda}$ decreases monotonically from $0$.\\
	Because $S$ is subject to  the pace distribution $Q_{\lambda}$ (\ref{Q}), according to Lemma \ref{TwoR}, the following holds  with confidence  $ 1 - \delta$
	\begin{align} \label{S2}
		R_{exp}(Q_{\lambda};f)-R_{emp}(S;f) \le \sqrt{\frac{\ln (1/\delta)}{n/2}}  .
	\end{align}
	Combining (\ref{S12}),(\ref{S1}) with (\ref{S2}), we have
	\begin{align*}
		R_{exp}(P_{target};f)   \le R_{emp}(S;f)+ e_{\lambda}  + \sqrt{\frac{\ln (1/\delta)}{n/2}}   .
	\end{align*}
	The proof is then completed.
\end{proof}
Finally, the proof of  Theorem \ref{theoremUp} is as follows.
\begin{proof}
	When we use training set $S=\{z_1,...,z_m\}$ to approximate $P_{train}$, we have
	\begin{align*}
		P_{train}=\sum_{i=1}^mp_i D_{z_i}(z), \ \forall i \in \{1,...,m\},
	\end{align*}
	where $p_i=\frac{1}{m}$ and $D_{z_i}(z)$ denotes  the Dirac delta function centered at $z_i$:
	\begin{align*} 
		\forall z!=z_i, \ D_{z_i}(z)=0 \ \text{and} \ \int_{Z} D_{z_i}(z) dz =1.
	\end{align*}
	It is easy to see that $P_{train}$ supposes a uniform density on each sample $z_i$.  Next, according to the special formulation (\ref{Q2}) of $Q_{\lambda}$,  we have:
	\begin{align*}
		Q_{\lambda}(z) \propto \sum_{i=1}^m  W_{\lambda}(z_{i})p_i D_{z_i}(z).
	\end{align*} 
	In this case, $S_{Q}=\{W_{\lambda}(z_1)z_1,...,W_{\lambda}(z_m)z_m  \}$  is  subject to the  pace distribution $Q_{\lambda}$. And the empirical risk on $Q_{\lambda}$ can be rewritten as 
	\begin{align*} 
		&R_{emp}(S_{Q};f)\\
		 =&\frac{1}{n_{\lambda}(n_{\lambda}-1)}\sum_{z_i,z_j \in S, z_i\neq z_j} W_{\lambda}(z_i) W_{\lambda}(z_j) L_f(z_i,z_j)
	\end{align*}
	where $n_{\lambda}=\sum_{i=1}^m \mathbb{I}(W_{\lambda}(z_i)!=0)$ denotes the number of selected samples in SPL setting. 
	
	And combining with  Lemma \ref{old}, we conclude that for any $\delta>0$ and any $f \in \mathcal{F}$, with confidence at least $1-\delta$ over a sample set $S$, we have: 
	\begin{small}
		\begin{align}
			R_{exp}(P_{target};f)  \le& \frac{1}{n_{\lambda}(n_{\lambda}-1)} \sum_{z_i,z_j \in S \atop z_i\neq z_j} W_{\lambda}(z_i) W_{\lambda}(z_j) L_f(z_i,z_j)   \nonumber \\
			+ & e_{\lambda}+\sqrt{\frac{\ln (1/\delta)}{n_{\lambda}/2}}, 
		\end{align}
	\end{small}
	where $n_{\lambda}$ denotes the number of selected samples,  $e_{\lambda}$ is defined as $R_{exp}(P_{target};f) - R_{exp}(Q_{\lambda};f)$ and  with the increasing of $\lambda$, $e_{\lambda}$ decreases monotonically from $0$.
\end{proof}

\section{Appendix D. Proof of Theorem \ref{Non-con}}
The objective function of our BSPAUC is defined as follows:
\begin{align}  \label{BSPAUCOF}
	&\min_{\theta,\mathbf{v},\mathbf{u}} \  \mathcal{L}(\theta,\mathbf{v},\mathbf{u};\lambda)   \\ 	=&\min_{\theta,\mathbf{v},\mathbf{u}} \   \frac{1}{nm} \sum_{i=1}^{n}\sum_{j=1}^{m}v_i u_j \xi_{ij} +\tau \Omega(\theta)  \\
	 -&  \lambda \left(\frac{1}{n}\sum_{i=1}^{n} v_i+\frac{1}{m}\sum_{j=1}^{m} u_j \right)+ \mu \left(\frac{1}{n}\sum_{i=1}^{n} v_i-\frac{1}{m}\sum_{j=1}^{m} u_j \right)^2   \nonumber   \\   
	& \ s.t.  \ \mathbf{v}\in [0,1]^n, \mathbf{u}\in [0,1]^m     \nonumber
\end{align}
where $\xi_{ij}=\max \{1-f(x^+_i)+f(x^-_j), 0 \}$ is the pairwise hinge loss.
\begin{proof}
	For the sake of clarity,  we define $\mathcal{K}(\mathbf{v},\mathbf{u})=\mathcal{L}(\mathbf{v},\mathbf{u};\theta,\lambda)$ as the sub-problem of (\ref{BSPAUCOF}) where $\theta$ and $\lambda$ are fixed:
	\begin{align}  \label{K}
		&\min_{\mathbf{v},\mathbf{u}} \ \mathcal{K}(\mathbf{v},\mathbf{u}) =\min_{\mathbf{v},\mathbf{u}} \  \mathcal{L}(\mathbf{v},\mathbf{u};\theta,\lambda)   \\ 	
		=&\min_{\mathbf{v},\mathbf{u}} \   \frac{1}{nm} \sum_{i=1}^{n}\sum_{j=1}^{m}v_i u_j \xi_{ij} 
		 -  \lambda \left(\frac{1}{n}\sum_{i=1}^{n} v_i+\frac{1}{m}\sum_{j=1}^{m} u_j \right) \nonumber\\
		 +& \mu \left(\frac{1}{n}\sum_{i=1}^{n} v_i-\frac{1}{m}\sum_{j=1}^{m} u_j \right)^2 + \text{const}  \nonumber   \\   
		& \ s.t.  \ \mathbf{v}\in [0,1]^n, \mathbf{u}\in [0,1]^m     \nonumber
	\end{align}
	
	First of all, we have to be clear that the necessary condition for $\mathcal{K}(\mathbf{v},\mathbf{u})$ to be a convex function is
	\begin{equation} \label{necessary}  \small
		\begin{aligned} 
			&\mathcal{K}(\frac{1}{2}(\mathbf{v}^1+\mathbf{v}^2),\frac{1}{2}(\mathbf{u}^1+\mathbf{u}^2))-\frac{1}{2}(\mathcal{K}(\mathbf{v}^1,\mathbf{u}^1)+\mathcal{K}(\mathbf{v}^2,\mathbf{u}^2)) \le 0,\\
			&\mathbf{v}^1,\mathbf{v}^2 \in [0,1]^n,  \mathbf{u}^1, \mathbf{u}^2 \in [0,1]^m. 
		\end{aligned}
	\end{equation} 
	Without loss of generality,  let 
	\begin{align*}
		\mathbf{v}^1=(1,1,...,1),\mathbf{u}^1=(0,0,...,0),\\
		\mathbf{v}^2=(0,0,...,0),\mathbf{u}^2=(1,1,...,1).
	\end{align*}
	Then $\mathcal{K}(\mathbf{v},\mathbf{u})$ satisfies
	\begin{align*}
		&\mathcal{K}(\frac{1}{2}(\mathbf{v}^1+\mathbf{v}^2),\frac{1}{2}(\mathbf{u}^1+\mathbf{u}^2))-\frac{1}{2}(\mathcal{K}(\mathbf{v}^1,\mathbf{u}^1)+\mathcal{K}(\mathbf{v}^2,\mathbf{u}^2)) \\
		=&\frac{1}{4nm} \sum_{i=1}^{n}\sum_{j=1}^{m} \xi_{ij}-\mu.
	\end{align*}
	Due to the condition that $\mu$ is a hyperparameter greater than 0 and  $\xi_{ij}$ is not less than 0, the formula $\frac{1}{4nm} \sum_{i=1}^{n}\sum_{j=1}^{m} \xi_{ij}-\mu$ is not guaranteed to be less than or equal to 0. And then according to the necessary condition (\ref{necessary}), we can conclude that if we fix $\theta$ in Eq. (\ref{BSPAUCOF}), the sub-problem   with respect to $\mathbf{v}$ and $\mathbf{u}$  maybe  non-convex.
\end{proof}

\section{Appendix E. Proof of Theorem \ref{solutionofvu}}
When we fix $\mathbf{u}$ and $\theta$, the sub-problem with respect to $\mathbf{v}$ is as follows:
\begin{small}
\begin{align}  \label{v}
	\mathbf{v}^{t}= & \argmin \limits_{\mathbf{v} \in [0,1]^n}   \ \mathcal{L}(\mathbf{v};\theta,\mathbf{u},\lambda)    \\
	=&\argmin \limits_{\mathbf{v} \in [0,1]^n} \ \frac{1}{n} \sum_{i=1}^{n}v_i l^+_i - \lambda \frac{1}{n}\sum_{i=1}^{n} v_i  + \mu\left(\frac{1}{n}\sum_{i=1}^{n} v_i-Q\right)^2 + \text{const},  \nonumber
\end{align}
\end{small}
where $l^+_i=\frac{1}{m}\sum_{j=1}^{m} u_j \xi_{ij}$ and  $Q=\frac{1}{m}\sum_{j=1}^{m} u_j$.

And when we fix $\mathbf{v}$ and $\theta$, the sub-problem with respect to $\mathbf{u}$ is as follows:
\begin{small}
\begin{align}  \label{u} 
	\mathbf{u}^{t}= &\argmin \limits_{\mathbf{u} \in [0,1]^m}  \  \mathcal{L}(\mathbf{u};\theta,\mathbf{v},\lambda)   \\
	=& \argmin \limits_{\mathbf{u} \in [0,1]^m}  \ \frac{1}{m} \sum_{j=1}^{m}u_j l^-_j - \lambda \frac{1}{m}\sum_{j=1}^{m} u_j  + \mu\left(\frac{1}{m}\sum_{j=1}^{m} u_j-P\right)^2 + \text{const},   \nonumber
\end{align}
\end{small}
where $l^-_j=\frac{1}{n}\sum_{i=1}^{n} v_i \xi_{ij}$ and $P=\frac{1}{n}\sum_{i=1}^{n} v_i$.

\begin{proof}
	Eq. (\ref{v}) can be expressed as
	\begin{align}  \label{simple}
		&\argmin \limits_{\mathbf{v} \in [0,1]^n}  \quad  \frac{1}{n} \sum_{i=1}^{n}v_i l^+_i - \lambda \frac{1}{n}\sum_{i=1}^{n} v_i  + \mu(\frac{1}{n}\sum_{i=1}^{n} v_i-Q)^2   \nonumber\\
		\iff &\argmin \limits_{\mathbf{v} \in [0,1]^n}  \quad   \sum_{i=1}^n v_ib_i+\frac{1}{n}(\sum_{i=1}^nv_i-nQ)^2  \\
		:= &\argmin \limits_{\mathbf{v} \in [0,1]^n}  \quad F(\mathbf{v})
	\end{align} 
	where $l^+_i=\frac{1}{m}\sum_{j=1}^{m} u_j \xi_{ij}$, $Q=\frac{1}{m}\sum_{j=1}^{m} u_j$ and $b_i=\frac{l_i^+ -\lambda}{\mu}$.\\
	Without loss of generality, we suppose that $b_1 \le b_2 \le \cdots \le b_n$. In this case, we firstly proof that the following $\mathbf{v} = (1,1,\ldots,1, v_p, 0,0,\ldots,0)$, which means that the $p$-th element $v_p$ is the only element in $\mathbf{v}$ that can not equal to 0 or 1,  minimizes $F(\mathbf{v})$.\\
	For each $\mathbf{v}\in[0,1]^n$, if $v_i<v_j$ for some $i<j$, then we can exchange $v_i$ and $v_j$, thus $\mathbf{v}$ becomes $\mathbf{v'}$. We have
	\begin{align*}
		&F(\mathbf{v'}) - F(\mathbf{v}) \\
		=& (b_iv_j +b_jv_i) - (b_iv_i +b_jv_j) = (b_i-b_j) (v_j-v_i) \leq 0. 
	\end{align*}
	Therefore, we can always assume that $1\geq v_1\geq v_2\geq \cdots \geq v_n \geq 0$  when $\mathbf{v}$ minimizes $F(\mathbf{v})$. \\
	Assuming that there exist $1>v_i>v_{i+1}>0$, we consider the following two cases.\\
	Case 1: $v_i+v_{i+1}\geq 1$. In this case, we replace $(v_i,v_{i+1})$ by $(1,v_i+v_{i+1}-1)$, then $\mathbf{v}$ becomes $\mathbf{v'}$. And we have
	\begin{align*}
		&F(\mathbf{v'}) - F(\mathbf{v}) \\
		= &(b_i +b_{i+1}(v_i+v_{i+1}-1)) - (b_iv_i +b_{i+1} v_{i+1}) \\
		= &(b_i-b_{i+1}) (1-v_i) \leq 0. 
	\end{align*}
	Case 2: $v_i+v_{i+1}< 1$. In this case, we replace $(v_i,v_{i+1})$ by $(v_i+v_{i+1},0)$, then $\mathbf{v}$ becomes $\mathbf{v'}$. And we have
	\begin{align*}
		F(\mathbf{v'}) - F(\mathbf{v})& = b_i (v_i+v_{i+1}) - (b_iv_i +b_{i+1} v_{i+1})\\
		& = (b_i-b_{i+1}) v_{i+1} \leq 0. 
	\end{align*}
	Therefore, we can always assume that at most one element in $\mathbf{v}$ is not equal to 0 or 1 when $\mathbf{v}$ minimizes $F(\mathbf{v})$. \\
	According to the fact that $1\geq v_1\geq v_2\geq \cdots \geq v_n \geq 0$ and at most one element in $\mathbf{v}$ is not equal to 0 or 1 when $\mathbf{v}$ minimizes $F(\mathbf{v})$, we conclude that  the following $\mathbf{v} = (1,1,\ldots,1, v_p, 0,0,\ldots,0)$,  which means that the $p$-th element $v_p$ is the only element in $\mathbf{v}$ that can not equal to 0 or 1,  minimizes $F(\mathbf{v})$.\\
	Then, we rewrite $F(\mathbf{v})$  as 
	\begin{align} \label{F}
		F(\mathbf{v}) = \frac{1}{n} ( v_p +p-1+\frac{nb_p}{2} -nQ   )^2 + \text{const}
	\end{align}
	and  consider the following three cases when $\mathbf{v} = (1,1,\ldots,1, v_p, 0,0,\ldots,0)$ minimizes $F(\mathbf{v})$.\\
	Case 1: $v_p=1$. According to Eq. (\ref{F}) we have that 
	$$
	p-1+\frac{nb_p}{2} -nQ \leq -1
	$$
	and thus 
	$$
	-\frac{b_p}{2} \geq \frac{p}{n} - Q.
	$$
	Furthermore, for each $i<p$ we have $v_i =1$ and 
	$$
	-\frac{b_i}{2} \geq -\frac{b_p}{2} \geq \frac{p}{n} - Q \geq \frac{i}{n} - Q  .
	$$
	And  for each $i>p$ we have $v_i =0$,  we rewrite $F(\mathbf{v})$  as
	$$
	F(\mathbf{v}) = \frac{1}{n} ( v_i +p+\frac{nb_i}{2} -nQ   )^2 + \text{const}.
	$$
	Since $\mathbf{v}$ minimizes $F(\mathbf{v})$, we have
	$$
	p+\frac{nb_i}{2} -nQ \geq 0
	$$
	and thus 
	$$
	-\frac{b_i}{2} \leq \frac{p}{n} - Q  \leq  \frac{i-1}{n} - Q  .
	$$
	Therefore, in this case we can get $\forall i \in\{1,2,...,n\}$,
	\begin{align}
		v_i=1 &\iff  -\frac{b_i}{2} \geq  \frac{i}{n} - Q  ,\\
		v_i=0 &\iff -\frac{b_i}{2} \leq   \frac{i-1}{n} - Q  .
	\end{align}
	Case 2: $0<v_p<1$. According to Eq. (\ref{F}) we have that 
	$$
	0< v_p= -(p-1+\frac{nb_p}{2} -nQ )<1
	$$
	and 
	thus 
	$$
	\frac{p-1}{n} - Q <-\frac{b_p}{2} < \frac{p}{n} - Q.
	$$
	Furthermore, for each $i<p$ we have $v_i =1$, we rewrite $F(\mathbf{v})$  as 
	$$
	F(\mathbf{v}) = \frac{1}{n} ( v_i +p-2+v_p+\frac{nb_i}{2} -nQ   )^2 + \text{const}.
	$$
	Since $\mathbf{v}$ minimizes $F(\mathbf{v})$, we have
	$$
	p-2+v_p+\frac{nb_i}{2} -nQ \leq -1 
	$$
	and thus 
	$$
	-\frac{b_i}{2} \geq  \frac{p-1+v_p}{n} - Q  \geq \frac{i}{n} - Q  .
	$$
	And  for each $i>p$ we have $v_i =0$,  we rewrite $F(\mathbf{v})$  as  
	$$
	F(\mathbf{v}) = \frac{1}{n} ( v_i +p-1+v_p+\frac{nb_i}{2} -nQ   )^2 + \text{const}.
	$$
	Since $\mathbf{v}$ minimizes $F(\mathbf{v})$, we have
	$$
	p-1+v_p+\frac{nb_i}{2} -nQ \geq 0
	$$
	and thus 
	$$
	-\frac{b_i}{2} \leq \frac{p-1+v_p}{n} - Q  \leq  \frac{i-1}{n} - Q  .
	$$
	Therefore, in this case we can get $\forall i \in\{1,2,...,n\} $,
	\begin{align}
		&v_i=1 \iff  -\frac{b_i}{2} \geq  \frac{i}{n} - Q  , \\
		&v_i=0 \iff -\frac{b_i}{2} \leq   \frac{i-1}{n} - Q,
	\end{align}
	and
	\begin{align}
		& 0< v_i =  -(i-1+\frac{nb_i}{2} -nQ ) < 1  \nonumber\\
		\iff& \frac{i-1}{n} - Q <-\frac{b_i}{2} < \frac{i}{n} - Q.
	\end{align}
	Case 3: $v_p=0$. According to Eq. (\ref{F}) we have that 
	$$
	p-1+\frac{nb_p}{2} -nQ \geq 0
	$$
	and thus 
	$$
	-\frac{b_p}{2} \leq \frac{p-1}{n} - Q.
	$$
	Similarly as Case 1, in this case, we can get $\forall i \in\{1,2,...,n\}$,
	\begin{align}
		v_i=1 &\iff  -\frac{b_i}{2} \geq  \frac{i}{n} - Q  , \\
		v_i=0 &\iff -\frac{b_i}{2} \leq   \frac{i-1}{n} - Q  .
	\end{align}
	By Cases 1, 2 and 3, it is easy to see that if $\mathbf{v}$ minimizes $F(\mathbf{v})$, then for each $1\leq i \leq n$, we always have
	\begin{align*}
		&v_i =0  \iff -\frac{b_i}{2} \leq  \frac{i-1}{n} - Q   , \\
		& v_i =1  \iff  -\frac{b_i}{2} \geq  \frac{i}{n} - Q ,
	\end{align*}
	and
	\begin{align*}
		&0< v_i =  -(i-1+\frac{nb_i}{2} -nQ ) < 1  \\
		\iff &\frac{i-1}{n} - Q <-\frac{b_i}{2} < \frac{i}{n} - Q  .\\
	\end{align*}
	which can be rewritten as:
\begin{equation}   \small
	\left \{\begin{array} {ll} 
		&v_p=1 \qquad   {\ \ \textrm{if}  \ \ }   l^+_p <  \lambda - 2\mu \left(\frac{p}{n}-\frac{\sum_{j=1}^m u_j}{m} \right)
		\\ &v_p=n  \left( \frac{\sum_{j=1}^m u_j}{m}-\frac{l^+_p - \lambda}{2 \mu}-\frac{p-1}{n} \right)  {\ \ \textrm{otherwise}  \ \ }  
		\\
		&v_p=0 \qquad   {\ \ \textrm{if}  \ \ }  l^+_p > \lambda - 2\mu \left(\frac{p-1}{n}-\frac{\sum_{j=1}^m u_j}{m} \right) 
	\end{array} \right.
\end{equation}
	where  $p \in \{1,...,n\}$ is the sorted index based on the loss values $l^+_p=\frac{1}{m}\sum_{j=1}^m u_j\xi_{pj}$.\\
	Similarly, we can get one global optimal solution of $\mathbf{u}$
		\begin{equation} \small
		\left \{\begin{array} {ll}
			&u_q=1  \qquad {\ \ \textrm{if}  \ \ }   l^-_q <  \lambda - 2\mu \left(\frac{q}{m}-\frac{\sum_{i=1}^n v_i}{n}\right) 
			\\
			&u_q=m \left(\frac{\sum_{i=1}^n v_i}{n}-\frac{l^-_q - \lambda}{2 \mu}-\frac{q-1}{m}\right) {\ \ \textrm{otherwise}  \ \ }  
			\\
			&u_q=0  \qquad {\ \ \textrm{if}  \ \ } l^-_q > \lambda - 2\mu \left(\frac{q-1}{m}-\frac{\sum_{i=1}^n v_i}{n}\right) 
		\end{array} \right.
	\end{equation}
	where  $q \in \{1,...,m\}$ is the sorted index based on the loss values $l^-_q=\frac{1}{n}\sum_{i=1}^n v_i\xi_{iq}$.
\end{proof}

\section{Appendix F. Proof of Theorem \ref{theormKstation}}

Firstly, we prove that $\mathcal{K}$ converges with the inner layer cyclic block coordinate descent procedure (\emph{i.e.}, lines 3-6 of Algorithm 1).
\begin{lemma} \label{theormKconverge}
	With the inner layer cyclic block coordinate descent procedure (\emph{i.e.}, lines 3-6 of Algorithm 1), the sub-problem $\mathcal{K}$ with respect to all weight variables converges.
\end{lemma}

\begin{proof}
	Firstly, we prove that $\mathcal{K}$  has a lower bound:
	\begin{equation}  \label{Kbound} \small
	\begin{aligned}
		\mathcal{K}&\overset{(a)}{\ge}  \frac{1}{nm} \sum_{i=1}^{n}\sum_{j=1}^{m}v_i u_j \xi_{ij} - \lambda\left(\frac{1}{n}\sum_{i=1}^{n} v_i+\frac{1}{m}\sum_{j=1}^{m} u_j\right) \\
		&\overset{(b)}{\ge} - \lambda\left(\frac{1}{n}\sum_{i=1}^{n} v_i+\frac{1}{m}\sum_{j=1}^{m} u_j\right) \overset{(c)}{\ge} -2\lambda_{\infty} > -\infty.
		\end{aligned}
	\end{equation}
	Inequality (a) follows from $ \text{const}=\tau \Omega(\theta) \ge 0$ and $ \mu(\frac{1}{n}\sum_{i=1}^{n} v_i-\frac{1}{m}\sum_{j=1}^{m} u_j)^2   \ge 0 $, and inequality (b) follows from $\xi_{ij} \ge 0$, $\mathbf{v}\in [0,1]^{n}$ and $\mathbf{u}\in [0,1]^{m}$. Considering that $\lambda_{\infty}$ is the maximum threshold of hyperparameter $\lambda$, when all $v_i$ and $u_j$ are equal to $1$, $ - \lambda(\frac{1}{n}\sum_{i=1}^{n} v_i+\frac{1}{m}\sum_{j=1}^{m} u_j) $ reaches its minimum value $-2\lambda_{\infty}$. Therefore, we can obtain the inequality (c). We prove that $\mathcal{K}$  has a lower bound. 
	
	Then, according to the update model of  inner layer cyclic block coordinate descent procedure of Algorithm 1, we need to solve the following two convex  sub-problems iteratively:
	\begin{align*}
		&\mathbf{v}^{k+1}= \argmin \limits_{\mathbf{v} \in [0,1]^n}  \mathcal{K}(\mathbf{v};\mathbf{u}^k),  \\
		&\mathbf{u}^{k+1}= \argmin \limits_{\mathbf{u} \in [0,1]^m}  \mathcal{K}(\mathbf{u};\mathbf{v}^{k+1}).
	\end{align*}
	Obviously, we have:
	\begin{align} \label{Keachstep}
		\mathcal{K}(\mathbf{v}^k,\mathbf{u}^k) \geq \mathcal{K}(\mathbf{v}^{k+1},\mathbf{u}^k) \geq \mathcal{K}(\mathbf{v}^{k+1},\mathbf{u}^{k+1}),
	\end{align}
	and we prove that  $\mathcal{K}$ does not increase at each update.
	
	Since $\mathcal{K}$ does not increase at each update  and it has a lower bound, we prove that with the inner layer cyclic block coordinate descent procedure (\emph{i.e.}, lines 3-6 of Algorithm 1), the sub-problem $\mathcal{K}$ with respect to all weight variables converges:
	\begin{align}  \label{Kdecreasing}
		&\lim\limits_{k \rightarrow \infty}\mathcal{K}(\mathbf{v}^{k},\mathbf{u}^{k})-\mathcal{K}(\mathbf{v}^{k+1},\mathbf{u}^{k})=0,  \\
		&\lim\limits_{k \rightarrow \infty}\mathcal{K}(\mathbf{v}^{k+1},\mathbf{u}^{k})-\mathcal{K}(\mathbf{v}^{k+1},\mathbf{u}^{k+1})=0.  \nonumber\\
		\Longrightarrow &\lim\limits_{k \rightarrow \infty}\mathcal{K}(\mathbf{v}^{k},\mathbf{u}^{k})-\mathcal{K}(\mathbf{v}^{k+1},\mathbf{u}^{k+1})=0. \nonumber 
	\end{align}
	The proof is then completed.
\end{proof}

Next, we introduce  the necessary definition and lemma.
\begin{definition} \label{definitionConstrainedF}
	\cite{bertsekas1997nonlinear,nouiehed2018convergence} For the  constrained optimization problem: $ \min_{\mathbf{x} \in \mathcal{F}} f(\mathbf{x}) $ where $\mathcal{F} \subseteq \mathcal{R}^n$ is a closed convex set.  A  point $\mathbf{x}^* \in \mathcal{F}$ is a first order stationary point when
	\begin{align*}
		\triangledown f(\mathbf{x}^*)' (\mathbf{x}-\mathbf{x}^*) \geq 0, \forall \mathbf{x}\in \mathcal{F}.
	\end{align*}
\end{definition}

\begin{lemma} \label{theoremLocalMinimum}
	\cite{bertsekas1997nonlinear}
	If $\mathbf{x}^*$ is a local minimun of $f$ over $\mathcal{F}$, then
	\begin{align*}
		\triangledown f(\mathbf{x}^*)' (\mathbf{x}-\mathbf{x}^*) \geq 0, \forall \mathbf{x}\in \mathcal{F}.
	\end{align*}
\end{lemma}
\begin{lemma} \label{TheoremCauchy}
	(Cauchy's convergence criterion) \cite{waner2001introduction}  The necessary and sufficient condition for sequence $\{X_n\}$ convergence is : For every $\varepsilon>0$,  there is a number $N$, such that for all $n, m > N$ holds
	$$|X_n -X_m| \leq \varepsilon.$$
\end{lemma}

Finally,  we prove that combining with the closed-form solutions  provided in Theorem \ref{solutionofvu},  $\mathcal{K}$  converges to a stationary point in Theorem \ref{theormKstation}.

\begin{proof}
	When we talk about $\mathbf{v}$, we suppose that
	\begin{align} \label{vnotconverge}
		\lim\limits_{k \rightarrow \infty}||\mathbf{v}^{k+1}-\mathbf{v}^{k}||^2 \geq \varepsilon>0,
	\end{align}
	where $k$ is the  iteration number of inner layer cyclic block coordinate descent procedure (\emph{i.e.}, lines 3-6 of Algorithm 1).  Same as Eq. (\ref{Kdecreasing}), we have:
	\begin{align} \label{Kconverge}
		\lim\limits_{k \rightarrow \infty}\mathcal{K}(\mathbf{v}^{k};\mathbf{u}^{k})-\mathcal{K}(\mathbf{v}^{k+1};\mathbf{u}^{k})=0.
	\end{align}
	According to the fact that $\mathbf{v}^{k+1}$ is one  global optimal solution of the sub-problem $\mathcal{K}(\mathbf{v};\mathbf{u}^{k})$, we have that  $\mathbf{v}^{k}$ in Eq. (\ref{Kconverge}) is also one global optimal solution of the sub-problem $\mathcal{K}(\mathbf{v};\mathbf{u}^{k})$. And according to the supposition (\ref{vnotconverge}), we have that $\mathbf{v}^k$ and  $\mathbf{v}^{k+1}$ are two different points.  However, when we update $\mathbf{v}$ with the Eq. (\ref{solutionofv}) in Theorem \ref{solutionofvu}, we can only own one global optimal solution of the sub-problem $\mathcal{K}(\mathbf{v};\mathbf{u}^{k})$ and thus the supposition does not hold. We prove that $\lim\limits_{k \rightarrow \infty}||\mathbf{v}^{k+1}-\mathbf{v}^{k}||^2=0$. Note that the proof of $\lim\limits_{k \rightarrow \infty}||\mathbf{u}^{k+1}-\mathbf{u}^{k}||^2=0$  is similar and then we have   
	\begin{align} \label{VUToPoint}
		\lim\limits_{k \rightarrow \infty}||(\mathbf{v}^{k+1},\mathbf{u}^{k+1})-(\mathbf{v}^{k},\mathbf{u}^{k})||^2=0.
	\end{align}
	Then, according to Lemma \ref{TheoremCauchy}, we have that there exists a  limit point $(\mathbf{v}^*,\mathbf{u}^*)$ of the sequence $\{(\mathbf{v}^k,\mathbf{u}^k)\}$, which satisfies:
	\begin{align*}
		\lim \limits_{k \rightarrow \infty}(\mathbf{v}^k,\mathbf{u}^k)=(\mathbf{v}^*,\mathbf{u}^*).
	\end{align*}
	And then, according to the update model of  inner layer cyclic block coordinate descent procedure of Algorithm 1 :
	\begin{align*}
		&\mathbf{v}^{k+1}= \argmin \limits_{\mathbf{v} \in [0,1]^n}  \mathcal{K}(\mathbf{v};\mathbf{u}^k)  \\
		&\mathbf{u}^{k+1}= \argmin \limits_{\mathbf{u} \in [0,1]^m}  \mathcal{K}(\mathbf{u};\mathbf{v}^{k+1})
	\end{align*}
	we have
	\begin{align}
		&\mathcal{K}(\mathbf{v}^*;\mathbf{u}^*) \le \mathcal{K}(\mathbf{v};\mathbf{u}^*) \ \forall \mathbf{v} \in [0,1]^n, \label{globalV}\\
		&\mathcal{K}(\mathbf{u}^*;\mathbf{v}^*) \le \mathcal{K}(\mathbf{u};\mathbf{v}^*) \ \forall \mathbf{u} \in [0,1]^m. \label{globalu}
	\end{align}
	According to Eq. (\ref{globalV}), we have that $\mathbf{v}^*$ is one global minimum solution of $\mathcal{K}(\mathbf{v};\mathbf{u}^*)$. Then, according to Lemma \ref{theoremLocalMinimum}, we have that:
	\begin{align*}
		\triangledown_{\mathbf{v}}\mathcal{K}(\mathbf{v}^*)'(\mathbf{v}-\mathbf{v}^*) \geq 0, \forall \mathbf{v} \in [0,1]^n
	\end{align*}
	where $\triangledown_{\mathbf{v}}\mathcal{K}$ denotes the gradient of $\mathcal{K}$ with respect to the block $\mathbf{v}$. Similarly, we have that
	\begin{align*}
		\triangledown_{\mathbf{u}}\mathcal{K}(\mathbf{u}^*)'(\mathbf{u}-\mathbf{u}^*) \geq 0, \forall \mathbf{u} \in [0,1]^m.
	\end{align*}
	Then, combining with the above  two inequalities, we have that
	\begin{align*}
		&\triangledown\mathcal{K}(\mathbf{v}^*,\mathbf{u}^*)'\left((\mathbf{v},\mathbf{u})-(\mathbf{v}^*,\mathbf{u}^*)\right)  \\
		=&(\triangledown_{\mathbf{v}}\mathcal{K}(\mathbf{v}^*),\triangledown_{\mathbf{u}}\mathcal{K}(\mathbf{u}^*))'(\mathbf{v}-\mathbf{v}^*,\mathbf{u}-\mathbf{u}^*) \\
		=&\triangledown_{\mathbf{v}}\mathcal{K}(\mathbf{v}^*)'(\mathbf{v}-\mathbf{v}^*) + \triangledown_{\mathbf{u}}\mathcal{K}(\mathbf{u}^*)'(\mathbf{u}-\mathbf{u}^*) \\
		\geq & 0, \ \forall(\mathbf{v},\mathbf{u}) \in [0,1]^{n+m}.
	\end{align*}
	Finally, according to the Definition \ref{definitionConstrainedF}, we have the limit point $(\mathbf{v}^*,\mathbf{u}^*)$ is a   stationary  point and then  we  prove that 	with the inner layer cyclic block coordinate descent procedure (\emph{i.e.}, lines 3-6 of Algorithm 1), the sub-problem $\mathcal{K}$ with respect to all weight variables converges to a stationary point.
	
\end{proof}

\section{Appendix G. Proof of Theorem \ref{Converge}}
\begin{proof}
	Before we prove the convergence of our BSPAUC (Algorithm 1), we first show that the value of the objective function $\mathcal{L}(\theta,\mathbf{v},\mathbf{u};\lambda)$ does not increase in each iteration of our BSPAUC. Let $\mathbf{v}^{t},\mathbf{u}^{t},\theta^{t}$ and $\lambda^t$ indicate the values of $\mathbf{v},\mathbf{u},\theta$ and $\lambda$ in the $t$-th iteration of outer layer cyclic block coordinate descent procedure (\emph{i.e.}, lines 2-9 of Algorithm 1).
	
	Same as Ep. (\ref{Keachstep}), we have:
	\begin{align*}
		\mathcal{K}(\mathbf{v}^k,\mathbf{u}^k) \geq \mathcal{K}(\mathbf{v}^{k+1},\mathbf{u}^k) \geq \mathcal{K}(\mathbf{v}^{k+1},\mathbf{u}^{k+1}),
	\end{align*}
where $k$ is the  iteration number of inner layer cyclic block coordinate descent procedure (\emph{i.e.}, lines 3-6 of Algorithm 1).  As such, we can obtain the following inequality:
	\begin{align} \label{KDecrease}
		\mathcal{L}(\theta^{t},\mathbf{v}^{t+1},\mathbf{u}^{t+1 };\lambda^{t}) &\le \mathcal{L}(\theta^{t},\mathbf{v}^{t},\mathbf{u}^{t};\lambda^{t}).
	\end{align}
	And then, we  follow the  assumption on Algorithm 2 and Algorithm 3 about $\theta$:
	\begin{align} \label{ThetaDecrease}
		\mathcal{L}(\theta^{t+1},\mathbf{v}^{t+1},\mathbf{u}^{t+1 };\lambda^{t}) &\le \mathcal{L}(\theta^{t},\mathbf{v}^{t+1},\mathbf{u}^{t+1};\lambda^{t}) .
	\end{align}
	Since $\lambda_{t+1} \ge \lambda_{t} >0$ and $\mathbf{v}\in [0,1]^n,\mathbf{u} \in [0,1]^m $, we obtain
	\begin{align*}
		\mathcal{L}(\theta^{t+1},\mathbf{v}^{t+1},\mathbf{u}^{t+1 };\lambda^{t+1}) \le \mathcal{L}(\theta^{t+1},\mathbf{v}^{t+1},\mathbf{u}^{t+1};\lambda^{t}) .
	\end{align*}
	Combining with the above inequalities, we have
	\begin{align*}
		\mathcal{L}(\theta^{t+1},\mathbf{v}^{t+1},\mathbf{u}^{t+1 };\lambda^{t+1}) \le \mathcal{L}(\theta^{t},\mathbf{v}^{t},\mathbf{u}^{t};\lambda^{t}) .
	\end{align*}
	We prove that  $\mathcal{L}$ does not increase in each iteration of our BSPAUC. Similar to Eq. (\ref{Kbound}),  we can easily prove that    $\mathcal{L}$ has a lower bound:
\begin{equation}   \small
		\begin{aligned}
			\mathcal{L}&\ge \frac{1}{nm} \sum_{i=1}^{n}\sum_{j=1}^{m}v_i u_j \xi_{ij} - \lambda\left(\frac{1}{n}\sum_{i=1}^{n} v_i+\frac{1}{m}\sum_{j=1}^{m} u_j\right) \\
			&\ge - \lambda\left(\frac{1}{n}\sum_{i=1}^{n} v_i+\frac{1}{m}\sum_{j=1}^{m} u_j\right) \ge -2\lambda_{\infty} > -\infty.
		\end{aligned}
\end{equation}
 Then, we prove that BSPAUC  converges  along  with the increase of hyper-parameter $\lambda$.
\end{proof}

\section{Appendix H. Proof of Theorem \ref{ConvergeToStan}}
\begin{proof}
	When $\lambda$ reaches its maximum $\lambda_{\infty}$,  we obtain the fixed objective function: $\mathcal{L}(\theta,\mathbf{v},\mathbf{u};\lambda_{\infty}).$  Let $\mathbf{v}^{t},\mathbf{u}^{t}$ and $\theta^{t}$  indicate the values of $\mathbf{v},\mathbf{u}$ and $\theta$  in the $t$-th iteration of outer layer cyclic block coordinate descent procedure (\emph{i.e.}, lines 2-9 of Algorithm 1).

	 We first talk about the weight parameters $(\mathbf{v},\mathbf{u})$. According to  Theorem \ref{Converge} and the optimizing procedure (\ref{KDecrease}) of  Algorithm 1, we have  
	\begin{small}
	\begin{align} \label{VUconvergeT}
		\lim\limits_{t \rightarrow \infty}\mathcal{L}(\mathbf{v}^{t},\mathbf{u}^{t};\theta^t,\lambda_{\infty})-\mathcal{L}(\mathbf{v}^{t+1},\mathbf{u}^{t+1};\theta^t,\lambda_{\infty})=0,
	\end{align}
	\end{small} 
which implies that for $t \to \infty$ and $ \forall k$:
\begin{align} \label{Kconverge2}
\mathcal{K}(\mathbf{v}^{k};\mathbf{u}^{k})-\mathcal{K}(\mathbf{v}^{k+1};\mathbf{u}^{k})=0,
\end{align}
where $k$ is the  iteration number of inner layer cyclic block coordinate descent procedure (\emph{i.e.}, lines 3-6 of Algorithm 1). 

We suppose that
\begin{align} \label{vnotconvergeT}
	\lim\limits_{t \rightarrow \infty}||\mathbf{v}^{t+1}-\mathbf{v}^{t}||^2 \geq \varepsilon>0,
\end{align}
which implies that for $t \to \infty$, $\exists k,$
\begin{align} \label{vnotconvergeK}
	||\mathbf{v}^{k+1}-\mathbf{v}^{k}||^2 >0.
\end{align}
According to the fact that $\mathbf{v}^{k+1}$ is one  global optimal solution of the sub-problem $\mathcal{K}(\mathbf{v};\mathbf{u}^{k})$, we have that  $\mathbf{v}^{k}$ in Eq. (\ref{Kconverge2}) is also one global optimal solution of the sub-problem $\mathcal{K}(\mathbf{v};\mathbf{u}^{k})$. And according to the supposition (\ref{vnotconvergeK}), we have that $\mathbf{v}^k$ and  $\mathbf{v}^{k+1}$ are two different points.  However, when we update $\mathbf{v}$ with the Eq. (\ref{solutionofv}) in Theorem \ref{solutionofvu}, we can only own one global optimal solution of the sub-problem $\mathcal{K}(\mathbf{v};\mathbf{u}^{k})$ and thus the supposition (\ref{vnotconvergeT}) does not hold.  We prove that $\lim\limits_{t \rightarrow \infty}||\mathbf{v}^{t+1}-\mathbf{v}^{t}||^2=0$. Note that the proof of $\lim\limits_{t \rightarrow \infty}||\mathbf{u}^{t+1}-\mathbf{u}^{t}||^2=0$  is similar and then we have   
\begin{align} \label{VUPointConvergeT}
	\lim\limits_{t \rightarrow \infty}||(\mathbf{v}^{t+1},\mathbf{u}^{t+1})-(\mathbf{v}^{t},\mathbf{u}^{t})||^2=0.
\end{align}
	And then, we  consider the model parameter $\theta$. According to Theorem \ref{Converge} and the optimizing procedure (\ref{ThetaDecrease}) of  Algorithm 1, we have:
\begin{align*}
	\lim\limits_{t \rightarrow \infty}\mathcal{L}(\theta^t;\mathbf{v}^{t+1},\mathbf{u}^{t+1},\lambda_{\infty})-\mathcal{L}(\theta^{t+1};\mathbf{v}^{t+1},\mathbf{u}^{t+1},\lambda_{\infty})=0,
\end{align*}
where $t$ is the  iteration number of outer layer cyclic block coordinate descent procedure (\emph{i.e.}, lines 2-9 of Algorithm 1).
Combining with that $\theta^{t+1}$ is initialized with $\theta^{t}$ and is one stationary point  of $\mathcal{L}(\theta;\mathbf{v}^{t+1},\mathbf{u}^{t+1},\lambda_{\infty})$ obtained by gradient based method,  \emph{i.e.}, Algorithm 2 \cite{gu2019scalable} or Algorithm 3 \cite{TSAM}, we get 
\begin{align} \label{thetaConvergeT}
	\lim\limits_{t \rightarrow \infty}||\theta^{t+1}-\theta^{t}||^2=0.
\end{align} 

 Combining with the above two formulas (\ref{thetaConvergeT}) and (\ref{VUPointConvergeT}),  we have   $$\lim\limits_{t \rightarrow \infty}||(\theta^{t+1},\mathbf{v}^{t+1},\mathbf{u}^{t+1})-(\theta^{t},\mathbf{v}^{t},\mathbf{u}^{t})||^2=0.$$ 
Then, according to Lemma \ref{TheoremCauchy}, we have that there exists a limit point $(\theta^*,\mathbf{v}^*,\mathbf{u}^*)$ of the sequence $\{(\theta^t,\mathbf{v}^t,\mathbf{u}^t)\}$ satisfying:
	$$\lim \limits_{t \rightarrow \infty}(\theta^t,\mathbf{v}^t,\mathbf{u}^t)=(\theta^*,\mathbf{v}^*,\mathbf{u}^*).$$
	According to Theorem \ref{theormKstation},  the sub-problem $\mathcal{L}(\mathbf{v},\mathbf{u};\theta^*,\lambda_{\infty})$ with respect to all weight parameters  converges to one stationary point, thus we have 
	$$\triangledown_{(\mathbf{v},\mathbf{u})}\mathcal{L}(\mathbf{v}^*,\mathbf{u}^*;\theta^*,\lambda_{\infty})'((\mathbf{v},\mathbf{u})-(\mathbf{v}^*,\mathbf{u}^*)) \geq 0 $$
	for $\forall(\mathbf{v},\mathbf{u}) \in [0,1]^{n+m}.$ At the same time, with gradient based method to optimize $\theta$,  \emph{i.e.}, Algorithm 2 \cite{gu2019scalable} or Algorithm 3 \cite{TSAM}, the sub-problem $\mathcal{L}(\theta;\mathbf{v}^*,\mathbf{u}^*,\lambda_{\infty})$ with respect to model parameter converges to one stationary point, thus  we have 
	$$\triangledown_{\theta}\mathcal{L}(\theta^*;\mathbf{v}^*,\mathbf{u}^*,\lambda_{\infty})'(\theta - \theta^*) \geq 0 $$
	for any $\theta$. Then, combining with the above  two inequalities, we have that
	\begin{align*}
		 &\triangledown\mathcal{L}(\theta^*,\mathbf{v}^*,\mathbf{u}^*;\lambda_{\infty})'\left((\theta,\mathbf{v},\mathbf{u})-(\theta^*,\mathbf{v}^*,\mathbf{u}^*)\right) \\
		=&(\triangledown_{\theta}\mathcal{L}(\theta^*),\triangledown_{(\mathbf{v},\mathbf{u})}\mathcal{L}(\mathbf{v}^*,\mathbf{u}^*))'(\theta -\theta^*,(\mathbf{v},\mathbf{u})-(\mathbf{v}^*,\mathbf{u}^*))\\
		=&\triangledown_{\theta}\mathcal{L} (\theta^*;\mathbf{v}^*,\mathbf{u}^*,\lambda_{\infty})' (\theta - \theta^*) \\ &+\triangledown_{(\mathbf{v},\mathbf{u})}\mathcal{L} (\mathbf{v}^*,\mathbf{u}^*;\theta^*,\lambda_{\infty})' ((\mathbf{v},\mathbf{u})-(\mathbf{v}^*,\mathbf{u}^*))\\
		\geq&  0.
	\end{align*}

	Finally, according to the Definition \ref{definitionConstrainedF}, we have that our  BSPAUC  converges to a stationary point of $\mathcal{L}(\theta,\mathbf{v},\mathbf{u};\lambda_{\infty})$ if the iteration number $T$ is large enough.
\end{proof}

\bibliography{reference}

\end{document}